\documentclass[twoside]{article}
\pdfoutput=1

\usepackage[accepted]{style_files/aistats2018}

\usepackage[utf8]{inputenc}
\usepackage[T1]{fontenc}
\usepackage[american]{babel}
\usepackage{url}
\usepackage{booktabs}
\usepackage{amsfonts}
\usepackage{nicefrac}
\usepackage{microtype}
\usepackage{csquotes}
\usepackage{enumitem}
\usepackage{etoolbox}
\usepackage{lmodern}
\usepackage[pdftex]{graphics}
\usepackage{amstext}
\usepackage{amsmath}
\usepackage{amssymb}
\usepackage{amsthm}
\usepackage{tikz}
\usepackage{graphicx}
\usepackage{pgfplots}
\usepackage{xcolor}
\usepackage{todonotes}
\usepackage{appendix}
\usepackage{epstopdf}
\usepackage{subcaption}
\usepackage[letterpaper, pass]{geometry}
\usepackage[colorlinks,linkcolor={red!80!black},citecolor={green!80!black},urlcolor={blue!80!black},hypertexnames=false]{hyperref}

\usepackage[noabbrev,capitalize]{cleveref}
\crefformat{equation}{(#2#1#3)}
\crefrangeformat{equation}{(#3#1#4) to~(#5#2#6)}
\crefmultiformat{equation}{(#2#1#3)}{ and~(#2#1#3)}{, (#2#1#3)}{ and~(#2#1#3)}
\crefname{appsec}{Appendix}{Appendices}

\usepackage{autonum}

\usepackage[backend=biber,style=authoryear,maxcitenames=1,maxbibnames=8,uniquelist=false,natbib,sortcites,isbn=false,doi=false]{biblatex}
\usepackage{utils/patch_biblatex}

\DeclareMathOperator*{\argmin}{argmin}
\DeclareMathOperator*{\arginf}{arginf}

\newcommand{\ud}{\mathrm d}
\DeclareMathOperator{\E}{\mathbb E}

\newcommand{\f}{\mathcal F}
\newcommand{\h}{\mathcal H}

\DeclareMathOperator{\N}{\mathcal N}

\newcommand{\p}{\mathcal P}
\newcommand{\R}{\mathbb R}
\newcommand{\n}{\mathbb N}
\newcommand{\PP}{\mathbb P}

\DeclareMathOperator{\supp}{supp}

\newcommand{\x}{\mathcal X}
\newcommand{\y}{\mathcal Y}

\makeatletter
\newcommand\given{\@ifstar{\mathrel{}\middle|\mathrel{}}{\mid}}

\DeclareRobustCommand{\abs}{\@ifstar\@abs\@@abs}
\newcommand{\@abs}[1]{\left\lvert #1 \right\rvert}
\newcommand{\@@abs}[1]{\lvert #1 \rvert}

\DeclareRobustCommand{\norm}{\@ifstar\@norm\@@norm}
\newcommand{\@norm}[1]{\left\lVert #1 \right\rVert}
\newcommand{\@@norm}[1]{\lVert #1 \rVert}

\DeclareRobustCommand{\inner}{\@ifstar\@inner\@@inner}
\newcommand{\@inner}[1]{\left\langle #1 \right\rangle}
\newcommand{\@@inner}[1]{\langle #1 \rangle}
\makeatother

\newtheorem{lemma}{Lemma}
\newtheorem{theorem}{Theorem}

\newlist{assumplist}{enumerate}{1}
\setlist[assumplist]{label=(\textbf{\Alph*})}
\Crefname{assumplisti}{Assumption}{Assumptions}

\newcommand\thetitle{Kernel Conditional Exponential Family}
\title{\thetitle}

\begin{document}

\begin{refsection}

\renewcommand*{\thefootnote}{\fnsymbol{footnote}}

\twocolumn[
\aistatstitle{\thetitle}

\aistatsauthor{Michael Arbel \And Arthur Gretton}
\aistatsaddress{%
  Gatsby Computational Neuroscience Unit, University College London\\
  \texttt{\{michael.n.arbel, arthur.gretton\}@gmail.com}}
]
\setlength\belowcaptionskip{-3ex}
\setcounter{footnote}{0}
\renewcommand*{\thefootnote}{\arabic{footnote}}
\begin{abstract}
	A nonparametric family of conditional distributions is introduced, which generalizes conditional exponential families  using functional parameters in a suitable RKHS. An algorithm is provided for learning the generalized natural parameter, and  consistency of the estimator is established in the well specified case. In experiments, the new method generally outperforms a competing approach with consistency guarantees, and is competitive with a deep conditional density model on datasets that exhibit abrupt transitions and heteroscedasticity.
\end{abstract}

\section{Introduction}

Distribution estimation is one of the most general problems in machine learning. Once an estimator for a distribution is learned, in principle, it allows to solve a variety of problems such as classification, regression, matrix completion and other prediction tasks.
With the increasing diversity and complexity of machine learning problems, regressing the conditional mean of $y$ knowing $x$ may not be sufficiently informative when the conditional density $p(y|x)$ is multimodal. In such cases, one would like to estimate the conditional distribution itself to get a richer characterization of the dependence between the two variables $y$ and $x$. In this paper, we address the problem of estimating conditional densities when $x$ and $y$ are continuous and multi-dimensional. 

Our conditional density model builds on a generalisation of the exponential family to infinite dimensions \citep{GuQiu93,Barron-91,Pistone-95,kernel-expfam,kenji:sieves}, where the natural parameter is a function in a reproducing kernel Hilbert space (RKHS): in this sense, like the Gaussian and Dirichlet processes, the kernel exponential family (KEF) is an infinite dimensional analogue of the finite dimensional case,  allowing to fit a much richer class of densities. 
While the maximum likelihood solution is ill-posed in infinite dimensions, \citet{Sriperumbudur:2013} have demonstrated that it is possible to fit the KEF via score matching \citep{Hyvarinen:2005}, which entails solving a linear system of size $n\times d$, where $n$ is the number of samples and $d$ is the problem dimension.
It is trivial to draw samples from such models using Hamiltonian Monte Carlo \citep{Neal:2012}, since they directly return the required potential energy \citep{Rasmussen2003,Strathmann:2015}.
In high dimensions, fitting a KEF model to samples becomes challenging, however: the computational cost rises as $d^3$, and complex interactions between dimensions can be difficult to model.

The complexity of the modelling task can be significantly reduced if a directed graphical model can be constructed over the variables, \citep{Pearl01,Jordan:1999}, where each variable depends on a subset of parent variables (ideally much smaller than the total, as in e.g.  a Markov chain).  In the present study, we extend the non-parametric family of \cite{Sriperumbudur:2013} to fit conditional distributions.
The natural parameter of the {\em conditional} infinite exponential family  is now an operator mapping the conditioning variable to a function space of features of the conditioned variable: for this reason, the score matching framework must be generalised to the vector-valued kernel regression setting of \cite{Micchelli:2005}.
We establish consistency in the well specified case by generalising the arguments of \cite{Sriperumbudur:2013} to the vector-valued RKHS. While our proof allows for general vector-valued RKHSs, we provide a practical implementation for a specific case, which takes the form of a linear system of size $n\times d$ \footnote{\texttt{The code can be found at: https://github.com/MichaelArbel/KCEF}}.

A number of alternative approaches have been proposed to the problem of conditional density estimation. \cite{Sugiyama:2010} introduced the
 Least-Square Conditional Density Estimation (LS-CDE) method, which provides an estimate  of a conditional density function $p(y|x)$ as a non-negative linear combination of basis functions. The method is proven to be consistent, and works well on reasonably complicated learning problems, although the optimal choice of basis functions for the method is an open question (in their paper, the authors use Gaussians centered on the samples).   Earlier non-parametric methods such as variants of Kernel Density Estimation (KDE) may also be used in conditional density estimation \citep{Fan:1996,Hall:1999}. These approaches also have consistency guarantees, however their performance degrades in high-dimensional settings \citep{Nagler:2016}. \cite{Sugiyama:2010} found that kernel density approaches performed less well in practice than LS-CDE.

 Kernel Quantile Regression (KQR), introduced by  \citep{Takeuchi:2006,Zhang:2016b}, allows to predict a percentile of the conditional distribution when $y$ is one-dimensional. KQR is formulated as a convex optimization problem with a unique global solution, and the entire solution path with respect to the percentile parameter can be computed efficiently \citep{Takeuchi:2009}. However, KQR applies only to one-dimensional outputs and, according to \cite{Sugiyama:2010}, the solution path tracking algorithm tends to be numerically unstable in practice.

 It is possible to represent and learn conditional probabilities without specifying probability densities, via conditional mean embeddings \citep{SonBooSidGoretal10,GruLevBalPatetal12}. These are conditional expectations of (potentially infinitely many) features in an RKHS, which can be used in obtaining conditional expectations of functions in this RKHS. Such expected features are complementary to the infinite dimensional exponential family, as they can be thought of as conditional expectations of an infinite dimensional sufficient statistic. This statistic can completely characterise the conditional distribution if the feature space is sufficiently rich \citep{SriGreFukSchetal10}, and has consistency guarantees under appropriate smoothness assumptions. Drawing samples given a conditional mean embedding can be challenging: this is possible via the Herding procedure \citep{CheWelSmol10,BacLacObo12}, as shown in \citep{KanNisGreFuk16}, but   requires a non-convex optimisation procedure to be solved for each sample.

 A powerful and recent deep learning approach to modelling conditional densities is the Neural Autoregressive Network (\cite{Uria:2013}, \cite{Raiko:2014} and \cite{Uria:2016}). These networks can be thought of as a generalization of the Mixture Density Network introduced by \cite{Bishop:2006}.  In brief, each variable is represented by a mixture of Gaussians, with means and variances depending on the parent variables through a deep neural network.  The network is  trained on  observed data  using stochastic gradient descent.
 Neural autoregressive networks  have shown their effectiveness for a variety of practical cases and learning problems.
Unlike the earlier methods cited, however, consistency is not guaranteed, and these methods require non-convex optimization, meaning that locally optimal solutions are found in practice.

We begin our presentation in  \cref{sec:expFamiliesScoreMatching}, where we  briefly present the Kernel Exponential Family. We generalise this model to the conditional case, in our first major contribution: this requires the introduction of vector-valued RKHSs and associated concepts. We then show that a generalisation of score matching may be used to fit the conditional density models for general vector valued RKHS, subject to appropriate conditions. We call this model the kernel conditional exponential family (KCEF).

Our second contribution, in \cref{sec:kernel_ridge_estimator}, is an empirical estimator for the \textit{natural parameter} of the KCEF (\cref{thm:main_estimator}), with convergence guarantees in the well-specified case (\cref{thm:main_rates}). In our experiments (\cref{sec:experiments}), we empirically validate the consistency of the estimator and compare it to other methods of conditional density estimation. Our method generally outperforms the leading alternative with consistency guarantees (LS-CDE). Compared with the deep approach (RNADE) which lacks consistency guarantees, our method has a clear advantage at small training sample sizes while being competitive at larger training sizes.

\section{Kernel Exponential Families}\label{sec:expFamiliesScoreMatching}
In this section we first present the kernel exponential family, which we then  extend  to a class of conditional exponential families. Finally, we provide a methodology for unnormalized density estimation within this class.
\subsection{Kernel Conditional Exponential Family}
\label{sec:KEF}
We consider the task of estimating the density $p(y)$ of a random variable $Y$ with support $\mathcal{Y}\subseteq\R^d$  from i.i.d samples $(Y_i)_{i=1}^n$.
 We propose to use a family of densities  parametrized by functions belonging to a Reproducing Kernel Hilbert Space (RKHS) $\h_{\y}$ with positive semi-definite kernel $k$ \citep{kernel-expfam,kenji:sieves,Sriperumbudur:2013}.
This exponential family of density functions takes the form
\begin{align}
	\left\{ p(y) := q_0(y) \frac{\exp{ \langle f,k(y,.)\rangle_{\h_{\y}} }}{Z(f)} \bigg\vert f \in \mathcal{F}   \right\},
	\label{eq:KEF}
\end{align}
where $q_0$ is a base density function on $\mathcal{Y}$ and  $ \mathcal{F}$ is the set of functions in the RKHS space $\h_{\y}$ such that $  Z(f) := \int_{\mathcal{Y}} \exp{ \langle f,k(y,.)\rangle_{\h_{\y}} }  q_0(y) dx < \infty$. In what follows, we call this family the \textit{kernel exponential family} (KEF) by analogy to classical exponential family. $f$ plays the role of the natural parameter while $k(y,.)$ is the sufficient statistic. Note that with an appropriate choice of the base distribution $q_0$ and a finite dimensional RKHS $\h_{\y}$, one can recover any finite dimensional exponential family. 
When $\h_{\mathcal{Y}}$ is infinite-dimensional, however, the family can approximate a much  broader class of densities on $\R^d$: under  mild conditions, it is shown in \cite{Sriperumbudur:2013} that the KEF approximates all densities of the form $\{ q_0(y)\exp{ (f(y) - A) }| f \in C_0(\mathcal{Y}) \}$, $A$ being the normalizing constant and $C_0(\mathcal{Y})$ the set of continuous functions with vanishing tails. 

Given two subsets $\mathcal{Y}$ and $\mathcal{X}$ of $\R^d $ and $\R^p $ respectively, we now propose to extend the KEF to a family of conditional densities $p(y|x)$. We  modify equation \cref{eq:KEF} by making the function $f$ depend on the conditioning variable $x$. The parameter $f$ is a function of two variables $x$ and $y$, $f: \x\times \y \rightarrow \R $ such that $y \mapsto f(x,y)$ belongs to the RKHS $\h_{\y}$ for all $x$ in $\x$. In all that follows, we will denote by $T$ the mapping
\[
T \::\: \mathcal{X} \rightarrow \h_{\y}
\qquad
x \mapsto T_x
\]
 such that $T_x(y) =  f(x,y)$ for all $y$ in $\y$

We next consider how to enforce a smoothness requirement on $T$ to make the conditional density estimation problem well-posed. To achieve this, we will require that the mapping $T$
belongs to a vector valued RKHS $\h$: we now briefly review the associated theory, following \citep{Micchelli:2005}. A Hilbert space $ (\h, \langle ., .\rangle_{\h} ) $ of functions $ T:\mathcal{X} \rightarrow \h_{\y} $ taking values in a vector space $\h_{\y}$ is said to be a vector valued RKHS if for all $x\in \mathcal{X}$ and $ h \in \h_{\y} $, the linear functional $T \mapsto \langle h, T_x \rangle_{\h_{\y}} $ is continuous. 
The reproducing property for vector-valued RKHSs follows from this definition. By the Riesz representation theorem, for each $x\in \mathcal{X}$ and $ h \in \h_{\y} $, there exists a linear operator $\Gamma_x $ from $\h_{\y}$ to $\h$ such that for all $T \in \h$,
\begin{align}
	\langle h, T_x \rangle_{\h_{\y}} = \langle T, \Gamma_x h \rangle_{\h}
		\label{eq:reproducing_op}
\end{align} 
Considering the dual operator $\Gamma_x^{*}  $ from $\h$ to $\h_{\y}$, we also get 
\begin{align}
	\Gamma_x^{*}T = T_x .
\end{align}

We can define a vector-valued reproducing kernel by composing the operator $\Gamma_x$ with its dual,
\begin{align}
	\Gamma(x,x') = \Gamma_x^{*}\Gamma_{x^{'}},
\end{align}
where for all $x$ and $x^{'}$, $ \Gamma(x,x^{'})$ is a bounded linear operator from $\h_{\y}$ to $\h_{\y}$, i.e.,  $\Gamma(x,x^{'}) \in   \mathcal{L}(\h_{\y} )$. The space $\h$ is said to be generated by an operator valued reproducing kernel $\Gamma$. One practical choice for $\Gamma$ is to define it as: 
\begin{align}
\label{eq:diagonal_op_kernel}
	\Gamma(x, x^{'}) = k_{\mathcal{X}}(x, x^{'})I_{\h_{\y}}  && \forall x, x^{'} \in \mathcal{X},
\end{align}
where $ I_{\h_{\y}} $ the identity operator on $\h_{\y}$ and $k_{\mathcal{X}}$ is now a real-valued kernel which generates a real valued RKHS $\h_{\mathcal{X}}$ on $\x$ \citep[as in the conditional mean embedding; see][]{GruLevBalPatetal12}. A simplified form of the estimator of $T$ will be presented in \cref{sec:kernel_ridge_estimator} for this particular choice for $\Gamma$ and will be used in the experimental setup in \cref{sec:experiments}.

We will now express $T_x(y)$ in a convenient form that will allow to extend the KEF. 
For a given $x$, recalling that $T_x$ belongs to $\h_{\y}$, one can write $T_x(y) = \langle T_x, k(y,.)\rangle_{\h_{\y}}$
for all $y$ in $\mathcal{Y}$.
Using the reproducing property in \cref{eq:reproducing_op}, one further gets $T_x(y) = \langle T, \Gamma_x k(y,.)\rangle_{\h}$.
By considering the subset $\mathcal{T}$ of $\h$ such that for all $x$ in $\x$ the  integral 
\begin{align}
	Z(T_x) := \int_{\y} q_0(y)\exp(\langle T, \Gamma_x k(y,.)\rangle_{\h})\ud y < \infty
\end{align}
is finite, we define the \textit{kernel conditional exponential family} (KCEF) as the set of conditional densities
\begin{align}
\label{cond_exp_fam}
	\left\{ p_{T}(y|x):= q_0(y)\frac{\exp{ \langle T, \Gamma_x k(y,.)\rangle_{\h} }}{ Z(T_x) } \bigg\vert T\in \mathcal{T}  \right\}	.
\end{align}
Here $T$ plays the role of the natural parameter while  $ \Gamma_x k(y,.) $ is the sufficient statistic. When $T$ is restricted to be constant with respect to $x$, we recover the \textit{kernel exponential family} (KEF). The KCEF is therefore an extension of the KEF introduced in \cite{Sriperumbudur:2013}. It is also a special case of the family introduced in \cite{kernel-expfam}. In the latter, the inner product is given by $\langle T, \phi(x,y)\rangle_{\h}$ where $\phi$ is a general feature of $x$ and $y$. In the present work, $\phi$ has the particular form $\phi(x,y) = \Gamma_x k(y,.)$. This allows to further express $p_T(y|x)$ for a given $x$ as an element in a KEF with sufficient statistic $k(y,.)$, by using the identity  $\langle T,\Gamma_x k(y,.)\rangle_{\h} = \langle T_x, k(y,.)\rangle_{\h_{\y}}$.
This is desirable since $p_T(y|x)$ remains in the same KEF family as $x$ varies and only the natural parameter $T_x$ changes.
\subsection{Unnormalized density estimation}
\label{sec:score_matching}
Given i.i.d samples $(X_i, Y_i)_{i=1}^n$ in $ \x\times \y $ following a joint distribution $\pi(x) p_0(y|x)$, where $\pi$ defines a marginal distribution over $\x$ and $p_0(y|x) $ is a conditional density function, we are interested in estimating $p_0$ from the samples  $(X_i, Y_i)_{i=1}^n$. Our goal is to find the optimal conditional density $p_{T}$ in the KCEF that best approximates $p_0$. 
The intractability of the normalizing constant  $Z(T_x)$ makes maximum likelihood estimation difficult.  \cite{Sriperumbudur:2013} used a score-matching approach  (see \cite{Hyvarinen:2005}) to avoid this normalizing constant; in the case of the KCEF, however,  the score function between $\pi(x) p_0(y|x)$ and $\pi(x) p_T(y|x)$ contains additional terms that involve the derivatives of the $\log$-partition function $ \log Z(T_x) $ with respect to $x$. Instead, we now propose a different approach with a modified version of the score-matching objective.

We define the expected conditional score between two conditional densities $ p_0(y|x) $ and $q(y|x)$ under a marginal density $\pi$ on $x$ to be:

\begin{align}
	J(p_0|q) := \int_{\x} \pi(x)  \mathcal{J}(p_0(.|x) || q(.|x)) dx
\end{align}
where:
\begin{align}
\mathcal{J}(p_0(.|x) \Vert q(.|x)) = \frac{1}{2}
\int_{\mathcal{Y}} p_0(y|x) \left\Vert \nabla_y \log \frac{p_0(y|x)}{q(y|x)} \right\Vert^2 dy
.\end{align}

For a fixed value $x$ in $\x$,  $\mathcal{J}(p_0(.|x) || q(.|x))$ is the score-matching function between $p_0(.|x) $ and $q(.|x) $ as defined by \cite{Hyvarinen:2005}. We further take the expectation over $x$ to define a divergence over conditional densities. 
The normalizing constant of  $q(y|x)$, which is a function of $x$,  is never involved in this formulation, as we take the gradient of the log-densities over $y$ only.
For a conditional density $p_0(y|x)$ that is supported on the whole domain $\y$ for all $x$ in $\x$, the expected conditional score is well behaved in the sense that $ J(p_0|q)$ is always non-negative, and reaches $0$ if and only if the two conditional distributions $p_0(y|x)$ and $q(y|x)$ are equal  for $\pi$-almost all $x$. More discussion can be found in \cref{app:failure_case}, for the case when this condition fails to hold.
The goal is then to find a conditional distribution $p_T$ in the KCEF for a given $T\in \mathcal{T}$ that minimizes this score over the whole family.

Under mild regularity conditions on the densities \citep[see][and below]{Hyvarinen:2005,Sriperumbudur:2013}, the score can be rewritten
\begin{align}
\label{score_2}
J(p_0||p_T) =&\E\Big[ \sum_{i=1}^d  \partial_i^2 T_x(y) + \frac{1}{2} (\partial_i T_x(y))^2 \Big] \\
+& \E\Big[ \sum_{i=1}^d \partial_i T_x(y) \partial_i\log q_0(y) \Big] + J(p_0||q_0)
\end{align}
where $J(p_0 || q_0)$ is a constant term for the optimization problem and the expectation is taken over $ \pi(x)p_0(y|x)$. All derivatives are with respect to $y$, and we used the notation $\partial_i f(y) =  \frac{\partial}{\partial y_i}f(y) $.
In the case of KCEF,  conditions to obtain this expression are satisfied  under assumptions in \cref{app:Assumptions}, as proved in \cref{thm:score_match} of \cref{app:score_match} .
The expression is further simplified using the reproducing property for the derivatives of functions in an RKHS (\cref{lemm:reprodcing_derivatives} of \cref{app:auxiliary}),
\begin{align}
	\partial_i T_x(y) &=  \langle T , \Gamma_{x}\partial_i k(y,.) \rangle_{\h}\\
	\partial_i^2 T_x(y) &=   \langle T , \Gamma_{x}\partial_i^2 k(y,.) \rangle_{\h}
\end{align}
which leads to:
 \begin{align}
 \label{score_final}
 	J(T) =& \E\Big[\sum_{i=1}^d  \frac{1}{2} \big\langle T , \Gamma_{x}\partial_i k(y,.) \big\rangle^2_{\h} + \big\langle T , \xi_i(x,y) \big\rangle_{\h} \Big]
 \end{align}
 with:
  \begin{align}
  \label{eq:xi_i}
 \label{xi}
\xi_i(x,y) =  \Gamma_{x}(\partial_i^2 k(y,.) + \partial_i\log q_0(y)\partial_i k(y,.)).
 \end{align}
  We introduced the notation  $J(T) := J(p_0||p_T) - J(p_0||q_0)$ for convenience.
 This formulation depends on $p_0(y|x)$ only through an expectation, therefore a Monte Carlo estimator of the score can be derived as a quadratic functional of $T$ in the RKHS  $ \h $,
 \begin{align}
 	\hat{J}(T) =& \frac{1}{n}\sum_{\substack{ b\in [n]\\i\in[d]}} \frac{1}{2} \big\langle T , \Gamma_{X_b}\partial_i k(Y_b,.) \big\rangle^2_{\h}  + \langle T , \xi_i(X_b, Y_b) \big\rangle_{\h}.
 \end{align}
Note that the objective functions $J(T)$ and $\hat{J}(T)$ can be defined over the whole space $\h$, whereas $J(p_0 | p_T)$ is meaningful only if $T$ belongs to $\mathcal{T}$.

\section{Empirical KCEF and consistency}
\label{sec:kernel_ridge_estimator}
In this section,  we will first estimate the optimal $T^{*} =  \argmin_{T\in\h} J(T) $ over the whole space $\h$ by minimizing a regularized version of the quadratic form in equation $\hat{J}(T)$, then we will state conditions under which all of the obtained solutions belong to $\mathcal{T}$ defining therefore conditional densities in the KCEF. 

Following \cite{Sriperumbudur:2013}, we define the kernel ridge estimator to be $ T_{n,\lambda} = \argmin_{T\in \h} \hat{J}(T) + \frac{\lambda}{2} \Vert T \Vert^2_{\h} $ where $ \Vert T \Vert_{\h} $ is the RKHS norm of $T$. $T_{n,\lambda}$ is then obtained by solving a linear system of $nd$ variables as shown in the next theorem:
\begin{theorem}
\label{thm:main_estimator}
Under assumptions listed in \cref{app:Assumptions}, and in particular if $\norm{\Gamma(x,x)}_{Op}$ is uniformly bounded on $\x$ for the operator norm, then the minimizer $T_{n,\lambda}$ exists, is unique, and is given by
	\begin{align}
		T_{n,\lambda} = -\frac{1}{\lambda} \hat{\xi} + \sum_{\substack{ b\in [n];i\in[d]}} \beta_{(b,i)} \Gamma_{X_b}\partial_i k(Y_b,\cdot),
	\end{align}
where
	\begin{align}
		\hat{\xi} = \frac{1}{n} \sum_{\substack{ b\in [n];i\in[d]}} \xi_i(X_b,Y_b),
	\end{align}
	and $\xi_i$ are given by \cref{eq:xi_i}.
	 $\beta_{(b,i)}$ denotes the $(b-1)d+i$ entry of a vector $\beta $ in $\R^{nd}$, obtained by solving the linear system
	\begin{align}
	(G + n\lambda I )\beta = \frac{h}{\lambda}	,
	\end{align}
where $G$ is an $nd$ by $nd$ Gram matrix, and $h$ is a vector in $\R^{nd}$,
\begin{align}
		(G)_{(a,i), (b,j)} 
		=&
		\inner{\Gamma_{X_a}\partial_i k(Y_a,\cdot),  \Gamma_{X_b}\partial_j k(Y_b,\cdot)}_{\h}\\
		(h)_{(b,i)} 
			 =&
			  \langle \hat{\xi}, \Gamma_{X_b}\partial_i k(Y_a,\cdot)  \rangle_{\h}
.\end{align}
\end{theorem}
The result is proved in \cref{thm:estimator_T} of \cref{app:Estimator}.

For the particular choice of $\Gamma$ in \cref{eq:diagonal_op_kernel}, the estimator takes a simplified form
\begin{align}
	T_{n,\lambda}(x,y) =&  -\frac{1}{\lambda} \hat{\xi}(x,y) +  \sum_{\substack{ b\in [n]\\i\in[d]}}  \beta_{(b,i)} k_{\x}(X_b,x)\partial_i k(Y_b,y),
\end{align}
with
	\begin{align}
		\hat{\xi}(x,y) =& \frac{1}{n} \sum_{ b\in [n];i\in[d]} k_{\x}(X_b,x) \partial_i^2 k(Y_b,y)  \\
		&+ \frac{1}{n} \sum_{ b\in [n];i\in[d]} k_{\x}(X_b,x)\partial_i\log q_0(Y_b) \partial_i k(Y_b,y)
	\end{align}

The coefficients  $\beta$ are obtained by solving the same system $ (G + n\lambda I )\beta = \frac{h}{\lambda}	
 $, where $G$ and $h$ reduce to
\begin{align}
		(G)_{(a,i), (b,j)}
		=&
		 k_{\x}(X_a, X_b) \partial_{i}\partial_{j+d}k(Y_a, Y_b),\\
		(h)_{(b,i)}
			 =&
			 \partial_i\hat{\xi}(X_b,Y_b),
\end{align}
and all derivatives are taken with respect to $y$.

The above estimator generalizes the estimator in \cite{Sriperumbudur:2013} to conditional densities. In fact, if one choses the kernel $k_{\x}$ to be a constant kernel, then one exactly recovers the setting of the KEF.

This linear system has a complexity of $\mathcal{O}(n^3 d^3)$ in time and  $\mathcal{O}(n^2d^2)$ in memory, which can be problematic for higher dimensions $d$ as $n$ grows. However, in practice, if the goal is to estimate a density of the form $ p(x_1,...,x_d)$, one can use the general chain rule for distributions,
$p(x_1)p(x_2|x_1) .... p(x_d|x_1,...,x_{d-1})$, and estimate each conditional density $p(x_i|x_1,...,x_{i-1})$ using the  KCEF in \cref{cond_exp_fam}. This reduces the complexity of the algorithm to $\mathcal{O}(n^3 d)$. A reduction to the cubic complexity in the number of data points $n$ could be managed via a Nyström-like approximation  \citep{Sutherland:2017}.   

In the well-specified case where the true conditional density $p_0(y|x)$ is assumed to be in \cref{cond_exp_fam} (i.e. $p_0(y|x) = p_{T_0}(y|x)$), we analyze the parameter convergence of the estimator $T_{n,\lambda}$ to $T_0$ and the convergence of the corresponding density $p_{T_{n,\lambda}}(y|x)$ to the true density $p_0(y|x)$.
First, we consider the covariance operator $C$ of the joint feature $\Gamma_x k(y,\cdot)$ under the joint distribution of $x$ and $y$, as introduced in \cref{thm:score_match} of \cref{app:score_match}, and we denote by $ \mathcal{R}(C^{\gamma}) $ the range space of the operator $ C^{\gamma}$. We then have the following consistency result:

\begin{theorem}
	\label{thm:main_rates}
	 Let $\gamma > 0$ be a positive parameter  and define $\alpha = \max(\frac{1}{2(\gamma + 1)}, \frac{1}{4})\in (\frac{1}{4}, \frac{1}{2} ) $. Under the conditions  in \cref{app:Assumptions}, for $ \lambda = n^{-\alpha} $,
and if  $T_0 \in \mathcal{R}(C^{\gamma})$, then
	\begin{align}
	 \Vert T_{n,\lambda} - T_0 \Vert = \mathcal{O}_{p_0}( n^{-\frac{1}{2} + \alpha} ).
	\end{align}	
	Furthermore, if  $ \sup_{y\in\y}k(y,y) < \infty$, then
	\begin{align}
			KL(p_0|| p_{T_{n,\lambda}}) = \mathcal{O}_{p_0}( n^{-1 + 2\alpha} ).
	\end{align}
\end{theorem}

These asymptotic rates match those obtained for the unconditional density estimator in \cite{Sriperumbudur:2013}. The smoother the parameter $T_0$, the closer $\alpha$ gets to $\frac{1}{4}$, which in turns leads to a convergence rate in KL divergence of the order of $ \frac{1}{\sqrt{n}}$. The worst case scenario happens when the range-space parameter $\gamma$ gets closer to $0$, in which case convergence in KL divergence happens at a rate close to $ \frac{1}{n^{\gamma}}$.
A more technical formulation of this theorem  along with a proof is presented in \cref{app:consistency}  (see \cref{thm:rates_appendix,thm:KL_rates}).  

The regularity of the conditional density $p(y|x)$ with respect to $x$ is captured by the boundedness assumption on the operator valued kernel $ \Gamma $; i.e.,  $  \norm{\Gamma(x,x)}_{op} \leq \kappa $ for all $x \in \x$ in   \cref{assumption_gamma}. This assumption allows to control the variations of the conditional distribution $p(y|x)$  as $x$ changes. Roughly speaking, we may estimate the conditional density $p(y|x_0)$ at a given point $x_0$ from samples $(Y_i, X_i)$ whenever there are $X_i$  sufficiently close to $x_0$. 
The uniformly bounded kernel $\Gamma$ allows to express the objective function $J(T)$ as a quadratic form $ J(T) = \frac{1}{2} \inner{T, C T}_{\h} + \inner{T, \xi}_{\h} + c_0  $ for constant $c_0$, where $C$ is the covariance operator introduced in  \cref{thm:score_match}. Furthermore, this boundedness assumption ensures that $C$ is a  "well-behaved" operator, namely a positive semi-definite trace-class operator. The population solution of the regularized score objective is then  given by $ T_{\lambda} = (C+\lambda I)^{-1}CT_0 $ while the estimator  is given by: $\hat{T}_{\lambda, n} = - (\hat{C} + \lambda I)^{-1}\hat{\xi} $ where $\hat{C}$ and $\hat{\xi}$ are empirical estimators for $ C$ and $\xi$.

The proof of consistency makes use of ideas from \cite{Sriperumbudur:2013,Caponnetto2007}, exploiting the properties of trace-class operators. The main idea is to first control the error $\norm{T_0 - \hat{T}_{\lambda, n}}_{\h} $ by introducing the population solution $T_{\lambda}$,
\begin{align}
	\norm{T_0 - \hat{T}_{\lambda, n}}_{\h} \leq \norm{T_0 - T_{\lambda}}_{\h} + \norm{T_{\lambda}- \hat{T}_{\lambda, n}}_{\h} 
\end{align}
The first term $\norm{T_0 - T_{\lambda}}_{\h}$ represents the regularization error which is introduced by adding a regularization term $\lambda I$ to the operator $C$. This term doesn't depend on $n$, and can be shown to decrease as the amount of regularization goes to $0$ with a rate $\lambda^{min(1, \gamma)} $. 
The second term represents the estimation error due to the finite number of samples $n$. This term decreases as $n \rightarrow 0 $ but  also increases when $ \lambda \rightarrow 0 $, therefore a trade-off needs to be made between decreasing the first term $\norm{T_0 - T_{\lambda}}_{\h}$ by setting $\lambda \rightarrow 0 $ and keeping the term $\norm{T_{\lambda} - \hat{T}_{\lambda, n}}_{\h}$  small enough. Using decompositions similar to those of \cite{Sriperumbudur:2013,Caponnetto2007}, we apply concentration inequalities on the general Hilbert space $\h$ to get a probabilistic bound on the estimation error of order $ \mathcal{O}(\frac{1}{\lambda \sqrt{n}}) $.

Concerning the convergence in KL divergence, the requirement that the real-valued kernel $ k $ is bounded implies that $ \mathcal{T} $ is in fact equal to $\h$. Therefore, minimizing the expected score $J(p_{T_0}||p_T)$ is equivalent to minimizing the quadratic form $J(T)$ over the whole RKHS $\h$. Finally, the rates in $KL$ divergence are obtained from the error rate of $\hat{T}_{\lambda, n}$.
\section{Experiments}\label{sec:experiments}
We perform a diverse set of experiments, on both synthetic and real data, in order to validate our model empirically.
In all experiments, the data are centered and rescaled such that the standard deviation for every dimension is equal to $1$. Given $(X_1^{(n)}, ...,X_d^{(n)} )_{n=1}^N$ i.i.d. samples of dimension $d$ we are interested in approximating the joint distribution $p_0(X_1,...,X_d)$ of data using different methods:

$\bullet$ The \textbf{KEF} model from \cite{Sriperumbudur:2013} approximates $p$ by a distribution $p_f$ that belongs to the KEF \cref{eq:KEF} by minimizing the score loss between $p$ and $p_f$ to find the optimal parameter  $f$. 

$\bullet$ The \textbf{KCEF} model of \cref{thm:main_estimator} approximates $p$ by a distribution $\hat{p}$  that is assumed to factorize according to some Directed Acyclic Graphical model (DAG): $\hat{p}(X_1,...,X_d) = \Pi_{i=1}^d \hat{p}(x_i|x_{\pi(i)}) $ where  $\pi(i)$ are the parent nodes of $i$. Note that we do not necessarily make independence assumptions, as the graph can be fully connected. We will consider in particular two graphs, the Full graph (\textbf{F}) of the form $\hat{p}(X_1, ..., X_d)  =  \hat{p}(X_1)\Pi_{i=2}^d\hat{p}(X_i|X_1, ..., X_{i-1})$ and  the Markov graph  (\textbf{M}) of the form $\hat{p}(X_1, ..., X_d)  =  \hat{p}(X_1)\Pi_{i=2}^d\hat{p}(X_i| X_{i-1})$.
 	Each of the factors is assumed to belong to the KCEF in \cref{cond_exp_fam}, and is estimated independently from the others by minimizing the empirical loss $\hat{J}(T)$ to find the optimal operator $T_i$ such that $\hat{p}(X_i|X_{\pi(i)}) = p_{T_i}(X_i|X_{\pi(i)})$. 
 	
$\bullet$ The \textbf{Orderless RNADE} model in \cite{Uria:2016}, where we train a 2 Layer Neural Autoregressive model with $100$ units per layer.
	The model consists of a product of conditional densities of the form $ \Pi_{i=1}^d p( X_{o_i} | X_{o_{<i}},\theta, o ) $, where $ o $ is a permutation of the dimensions $ [1,...,d]$ and $\theta$ is a set of parameters that are shared across the factors regardless of the chosen permutation $o$. RNADE is trained by minimizing the empirical  expected  negative log-likelihood, where the expectation is taken over all possible permutations and data,
		      \begin{align}
		      	\mathcal{L}(\theta) = \E_{o\in D} \E_{X\in\R^d} \Big[- \log p( X_{o_i} | X_{o_{<i}},\theta, o )\Big].
		      \end{align}
$\bullet$ The \textbf{LSCDE} model in \cite{Sugiyama:2010} where we also used the 2 factorizations of the joint distribution (\textbf{F}, \textbf{M}) and solve a least-squares problem to estimate each  of the conditional densities. The approximate densities are of the form $ \alpha^T \phi(X_i, X_{\pi(i)})  $ where $\phi$ is a vector of $m$ known non-negative functions and $\alpha $ is obtained by minimizing the squared error between $ p(X_i, X_{\pi(i)})$ and $\alpha^T \phi(X_i, X_{\pi(i)}) $. Only the non-negative component of the solution   $\alpha$ is used. 

                      For all variants of our model, we take the base density $q_0$ to be a centered gaussian with a standard deviation of $2$. The kernel function used for both predicted variable  $y$ and conditioning variable $x$ is the anisotropic radial basis function (RBF) with per-dimension bandwidths.
                      		The bandwidths and the regularization parameter $\lambda$ are tuned by gradient descent on the  cross validated score.

{\bf Synthetic data:} We consider the 'grid' dataset, which is a $d$-dimensional distribution with a tractable density that factorizes in the  form
\begin{align}
	p(x_i|x_{i-1}) = C_i(1+\sin (2\pi w_i^{a} x_i )\sin(2\pi w_i^{b} x_{i-1} ) )
\end{align}
for all $i\in [d]$. $C_i$ is a tractable normalizing constant. Samples are generated using rejection sampling for each dimension. To study the effect of sample size on the estimator, we generate $n$ training points with $n$ varying from $200$ to $2000$ and $d = 3$, and estimate the log-likelihood on $2000$ newly generated points. To compare the effect of dimension, we generate $2000$ datapoints of dimension $d$ varying from $2$ to $20$, and estimate the log-likelihood on $2000$ test points. 
Unlike in \citep{Sriperumbudur:2013,Sutherland:2017}, the score function $\hat{J}(T)$ cannot be used as a metric to compare different factorizations of the estimated distribution, as it is dependent on the specific factorization of the joint distribution.  
Instead, we estimated the log-likelihood for our proposed model \textbf{KCEF}, where the normalizing constants are computed using importance sampling. We discarded the \textbf{KEF} in this experiment,  since estimating the normalizing constant in high dimensions becomes impractical.  

In \cref{fig:exp_wave}(left), we plot the log-likelihood as the number of samples increases.  Both variants of \textbf{KCEF} (\textbf{F}, \textbf{M}) performed slightly better than the other methods in terms of speed of convergence as sample size increases. The variants that exploit the Markov structure of data  \textbf{M} lead to the best performance for both  \textbf{KCEF} and \textbf{LSCDE} as expected. The \textbf{NADE} method has comparable performance for large sample sizes, but the performance drops significantly for small sample sizes. This behaviour will also be observed in
subsequent experiments on real data.
The figure on the right shows the evolution of the log-likelihood per dimension as the dimension increases. 
 In the  \textbf{F} case, our approach is comparable to \textbf{LSCDE} with an advantage in small dimensions. The  \textbf{F} approaches both use an anisotropic RBF kernel with tuned per-dimension bandwidth which end up performing a kind of automatic relevance determination. This helps getting comparable performance to the \textbf{M} methods. A drastic drop in performance can happen when an isotropic kernel is used instead as confirmed  by  \cref{fig:old_exp_wave} of  \cref{app:additional_exp}. Finally, \textbf{NADE}, which is also agnostic to the Markov structure of data, seems to achieve comparable performance to the \textbf{F} methods with a slight disadvantage in higher dimensions. 
  \begin{figure*}
\begin{minipage}{0.5\textwidth}
\includegraphics[width=\linewidth]{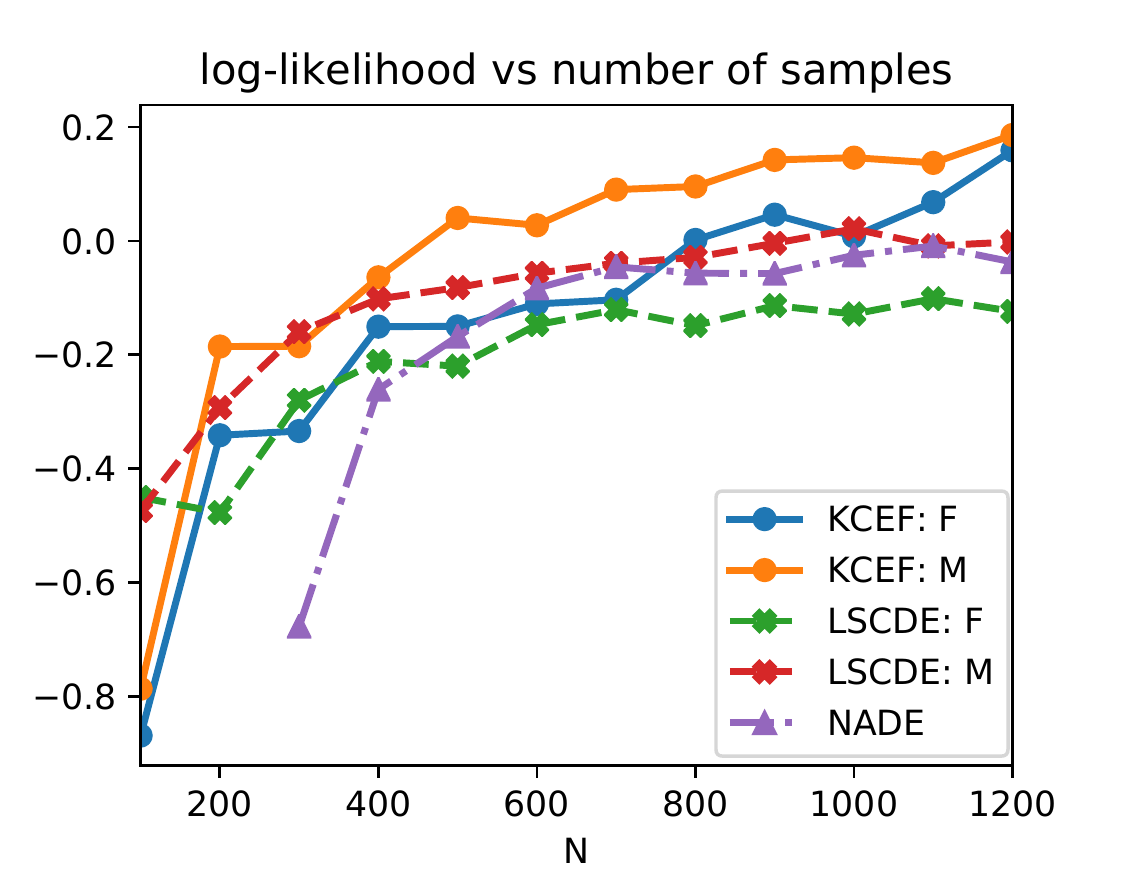}
\end{minipage}
\begin{minipage}{0.5\textwidth}
\includegraphics[width=\linewidth]{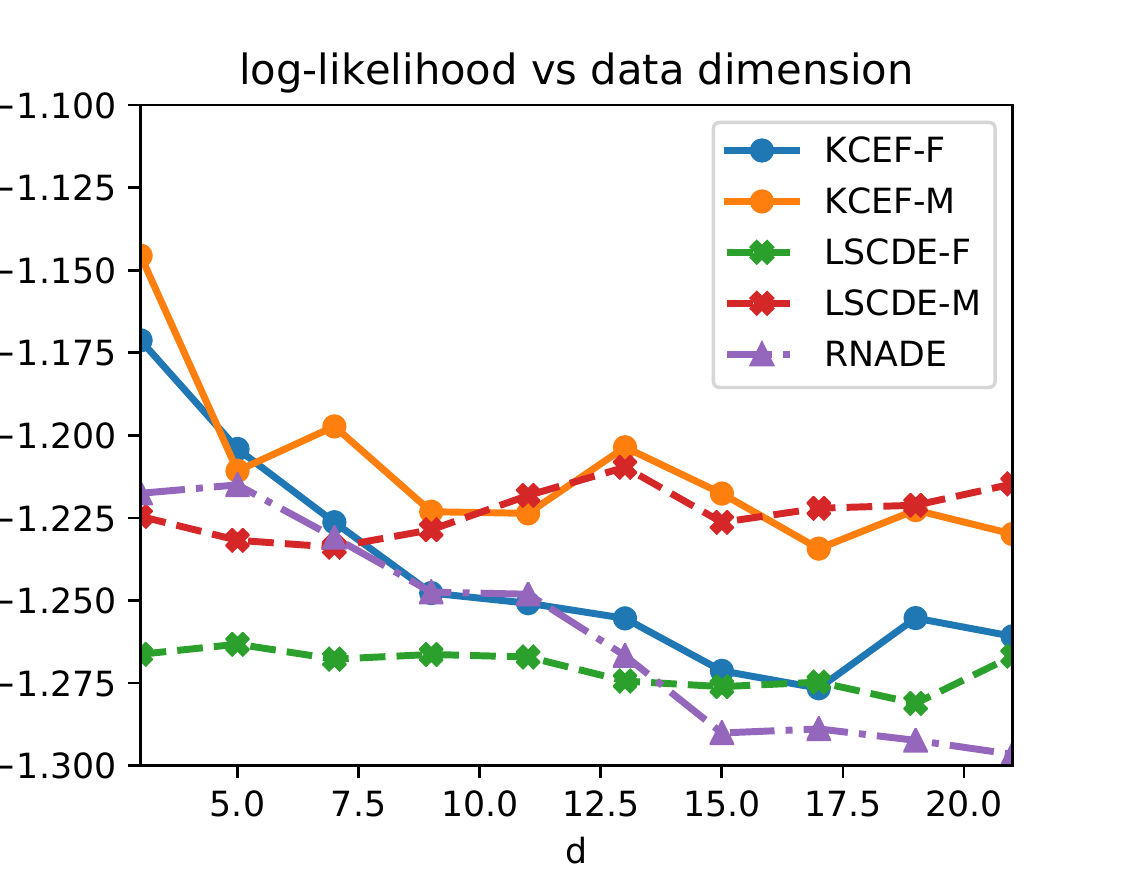}
\centering
\end{minipage}\hfill
\caption{Experimental comparison of proposed method KCEF and other methods (  LSCDE and NADE ) on synthetic \textit{grid} dataset. \textbf{LEFT}: log-likelihood vs  training samples size, $(d = 3)$. \textbf{RIGHT}:  log-likelihood per dimension vs dimension,  $N=2000$. The log-likelihood is evaluated on a separate test set of size $2000$.  }
\label{fig:exp_wave}
\end{figure*}\hfill

{\bf Real data:} We applied the proposed and existing methods to the  \textit{R} package benchmark datasets  \citep{Team:2008} (see \cref{tab:benchmark})  as well as three UCI datasets previously used to study the performance of other density estimators (see \cite{Uria:2013,Silva:2010,Tang:2012}). In all cases data are centered and normalized. 

First, the  \textit{R}  benchmark datasets are low dimensional with few samples, but with a relatively complex conditional dependence between the variables. This setting allows to compare the methods in terms of data efficiency and overfitting.
Each dataset was randomly split into a training and a test set of equal size. The models are trained to estimate the conditional density of a one dimensional variable $y$ knowing $x$  using samples $ (x_i,y_i)_{i=1}^n $ form the training set. The accuracy is measured by the negative log-likelihood for the test samples $ (\widetilde{x}_i, \widetilde{y}_i )_{i=1}^n $ averaged over $20$ random splits of data. 
We compared the proposed method with \textbf{NADE} and \textbf{LSCDE} on 14 datasets. 
 For \textbf{NADE} we used CV over the number of units per layer $\{2, 10, 100\}$ and number of mixture components $\{ 1, 2, 5, 10 \}$ for a 2 layer network. We also used CV to chose the hyper-parameters for \textbf{LSCDE} and the proposed method on a  $20\times 20$ grid (for $\lambda$ and $\sigma$).
 
 The experimental results are summmarized in \cref{tab:benchmark}. \textbf{LSCDE} worked well in general as claimed in the original paper, however the proposed method substantially improves the results. On the other hand, \textbf{NADE} performed rather poorly due to the small sample size of the training set, despite our attempts to improve its performance by reducing the number of parameters to train and by introducing early stopping.

 The UCI datasets (Red Wine, White Wine and Parkinsons) represent challenging %
 datasets with non-linear dependencies and abrupt transitions between high and low density regions. This makes the densities difficult to model using standard  tools such as mixtures of Gaussians or factor analysis. They also contain enough training sample points to allow a stronger performance by \textbf{NADE}. All discrete-valued variables were eliminated as well as one variable from every pair of variables that are highly correlated (Pearson correlation greater than 0.98). Following \cite{Uria:2013}, $90\%$ of the data were used for training while $10\%$ were held-out for testing. 
 Two different graph factorizations (\textbf{F}, \textbf{M} ) were used for the proposed method and for \textbf{LSCDE}.

In \cref{tab:UCI_dataset}, we report the performance of the different models. Our  method was among the statistically significant group of best models on Parkinsons dataset according to the two-sided paired $t-test$ at significance level of $5\%$. On the remaining datasets, it achieved the second best performance after \textbf{NADE}.
\begin{table}
\let\center\empty
\let\endcenter\relax
\centering
\resizebox{.8\width}{!}{\begin{tabular}{llll}
\toprule
{} &                       KCEF &              NADE &                      LSCDE \\
\midrule
caution       &  $\mathbf{ 0.99 \pm 0.01}$ &   $4.12 \pm 0.02$ &            $1.19 \pm 0.02$ \\
ftcollinssnow &   $\mathbf{ 1.46 \pm 0.0}$ &   $3.09 \pm 0.02$ &            $1.56 \pm 0.01$ \\
highway       &  $\mathbf{ 1.17 \pm 0.01}$ &  $11.02 \pm 1.05$ &            $1.98 \pm 0.04$ \\
heights       &   $\mathbf{ 1.27 \pm 0.0}$ &    $2.71 \pm 0.0$ &              $1.3 \pm 0.0$ \\
sniffer       &  $\mathbf{ 0.33 \pm 0.01}$ &   $1.51 \pm 0.04$ &            $0.48 \pm 0.01$ \\
snowgeese     &  $\mathbf{ 0.72 \pm 0.02}$ &    $2.9 \pm 0.15$ &            $1.39 \pm 0.05$ \\
GAGurine      &   $\mathbf{ 0.46 \pm 0.0}$ &   $1.66 \pm 0.02$ &             $0.7 \pm 0.01$ \\
geyser        &            $1.21 \pm 0.04$ &   $1.43 \pm 0.07$ &   $\mathbf{ 0.7 \pm 0.01}$ \\
topo          &  $\mathbf{ 0.67 \pm 0.01}$ &   $4.26 \pm 0.02$ &             $0.83 \pm 0.0$ \\
BostonHousing &    $\mathbf{ 0.3 \pm 0.0}$ &    $3.46 \pm 0.1$ &            $1.13 \pm 0.01$ \\
CobarOre      &            $3.42 \pm 0.03$ &    $4.7 \pm 0.02$ &  $\mathbf{ 1.61 \pm 0.02}$ \\
engel         &   $\mathbf{ 0.18 \pm 0.0}$ &   $1.46 \pm 0.02$ &            $0.76 \pm 0.01$ \\
mcycle        &  $\mathbf{ 0.56 \pm 0.01}$ &   $2.24 \pm 0.01$ &            $0.93 \pm 0.01$ \\
BigMac2003    &  $\mathbf{ 0.59 \pm 0.01}$ &   $13.8 \pm 0.13$ &            $1.63 \pm 0.03$ \\
\bottomrule
\end{tabular}
}
\caption{Mean and std. deviation of the negative log-likelihood on benchmark data over 20 runs, with different random splits. In all cases $d_y = 1$. Best method in boldface (two-sided paired \textit{t-test} at  $5\%$).}
\label{tab:benchmark}
\end{table}
\begin{table}
\let\center\empty
\let\endcenter\relax
\centering
\resizebox{.8\width}{!}{\begin{tabular}{llll}
\toprule
 &                 white -wine &                 parkinsons &                  red wine \\
\midrule
KCEF-F  &           $13.05 \pm 0.36$ &  $\mathbf{ 2.86 \pm 0.77}$ &           $11.8 \pm 0.93$ \\
KCEF-M  &           $14.36 \pm 0.37$ &            $5.53 \pm 0.79$ &          $13.31 \pm 0.88$ \\
LSCDE-F &            $13.59 \pm 0.6$ &           $15.89 \pm 1.48$ &           $14.43 \pm 1.5$ \\
LSCDE-M &           $14.42 \pm 0.66$ &           $10.22 \pm 1.45$ &          $14.06 \pm 1.36$ \\
NADE    &  $\mathbf{ 10.55 \pm 0.0}$ &             $3.63 \pm 0.0$ &  $\mathbf{ 9.98 \pm 0.0}$ \\
\bottomrule
\end{tabular}
}
\caption{UCI results: average and standard deviation of the negative log-likelihood over 5 runs with different random splits. Best method in boldface (two-sided paired \textit{t-test} at  $5\%$).}
\label{tab:UCI_dataset}
\end{table}

{\bf Sampling:} We compare samples generated from the approximate distribution obtained using different methods (\textbf{KEF}, \textbf{KCEF}, \textbf{NADE}). To get samples $(X_1,..., X_d)$ from the joint distribution of \textbf{KCEF} we performed ancestral sampling, where a sample from the parents $\pi(i)$ of node $i$ is first generated, and  then $X_i$ is sampled according to $ p(X_i | X_{\pi(i)}) $. We used the methodology and code in \cite{Strathmann:2015} to sample from each conditional distribution $p(X_i| X_{\pi(i)})$ using an HMC proposal,  since we have access to the gradient of the conditional densities and their un-normalized values. 
We trained the 3 models on Red Wine and Parkinsons datasets as described
previously,
and generated joint samples from two-dimensional slices of data (see \cref{fig:UCI_samples}).
Since each conditional distribution is low-dimensional, we assumed an idealized scenario where the burn-in is completed after $100$ iterations of the HMC sampler. We then run $20$ samplers for $1000$ and thin by a factor $10$, which results in $2000$ samples.
As shown in \cref{fig:UCI_samples}, \textbf{KCEF} is able to capture challenging properties of the target distribution, such as heteroscedasticity and sharp thresholds.
\begin{figure}
\centering
\includegraphics[width=\linewidth]{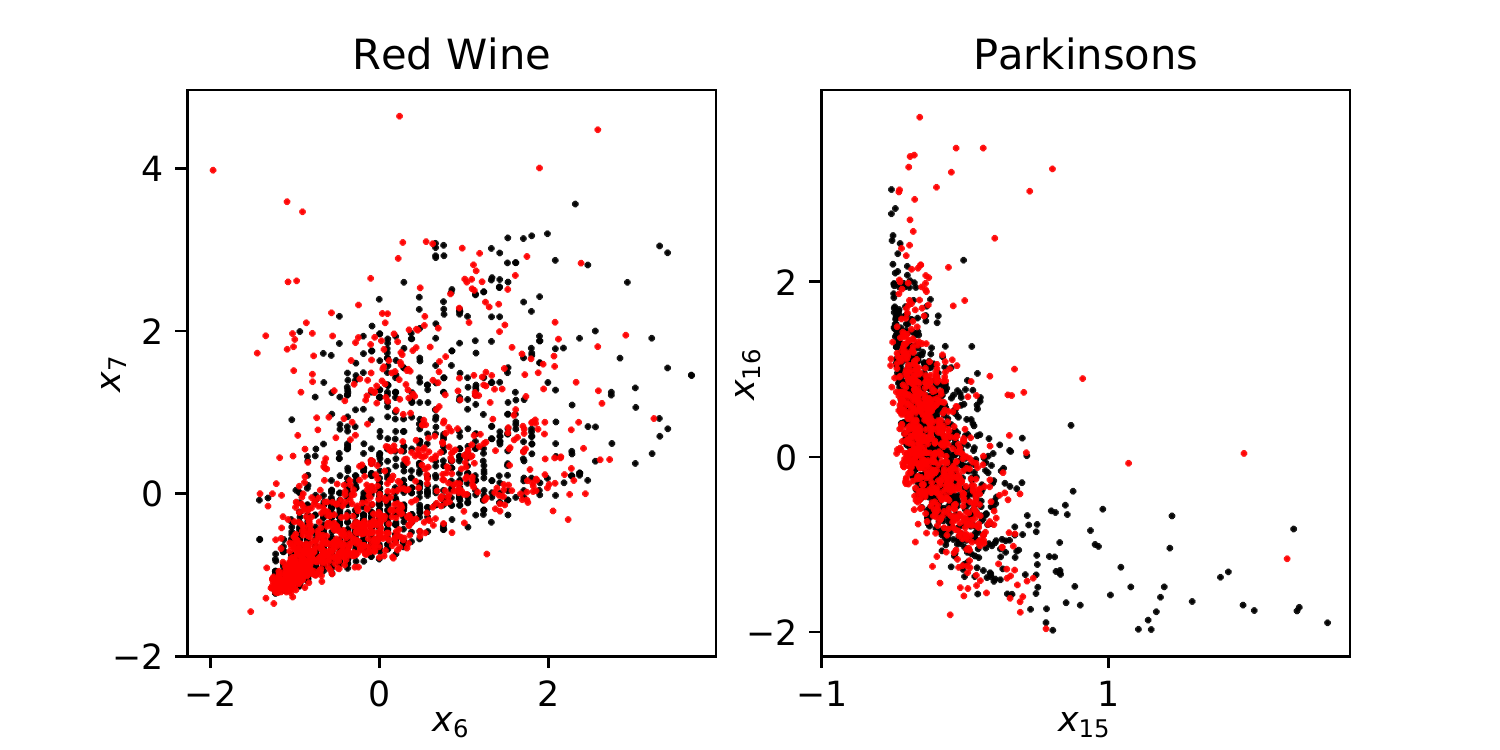}
\caption{Scatter plot of 2-d slices of \textit{red wine} and \textit{parkinsons} data sets.%
  Black points are real data, red are samples from the KCEF.}
\label{fig:UCI_samples}
\end{figure}
\begin{table}
\let\center\empty
\let\endcenter\relax
\centering
\resizebox{.8\width}{!}{\begin{tabular}{lrr}
\toprule
 &  $ \h_{KEF < KCEF} $  &  $ \h_{NADE < KCEF} $  \\
\midrule
parkinsons &                                         0.523506 &                                          0.011467 \\
red-wine   &                                         0.000791 &                                          0.326109 \\
\bottomrule
\end{tabular}
}
\caption{$p$-values for the relative similarity test. Columns represents the p-values for testing whether samples from KEF ( resp. KCEF) model are closer to the data than samples from the KCEF (resp. NADE).}
\label{tab:MMD}
\end{table}

We also performed a test of relative similarity between the generated samples and the ground truth data following the methodology and code of \cite{Bounliphone:2015}. Given samples from data $X_m$ and generated samples  $Y_n$ and $Z_r$ from two different methods, we test the  hypothesis that $ P_x $ is closer to  $P_z$ than $P_y$ according to the MMD metric.
The null hypothesis  
\begin{align}
	 \h_{y<z} : MMD(P_x, P_y) \leq  MMD(P_x, P_z)
\end{align}
  is tested against the alternative
  at a significance level $\alpha = 5\%$ (see \cite{Bounliphone:2015} for details). \cref{tab:MMD} shows the p-value for testing \textbf{KCEF} vs \textbf{KEF} and \textbf{NADE} vs \textbf{KCEF}.
  We see that \textbf{KCEF} significantly outperforms \textbf{NADE}  with high confidence for the \textit{parkinsons} dataset, consistent with \cref{tab:UCI_dataset}. Performance of the two methods is not statistically distinguishable for the \textit{red-wine} data.
See the scatter plots in \cref{fig:UCI_samples_all_methods} of \cref{app:additional_exp}, which visually confirm the result. 
  \textbf{KCEF} gives significantly better samples than \textbf{KEF} on \textit{red-wine}: indeed, \textbf{KCEF} generally outperforms   \textbf{KEF} on distributions where the densities exhibit  abrupt transitions, as is clear by inspection of the plots in  \cref{fig:UCI_samples_all_methods} of \cref{app:additional_exp}.

\printbibliography

\end{refsection}

\onecolumn
\begin{refsection}

\section*{Appendices}
\label{Appendix}

\begin{appendices}

\crefalias{section}{appsec}
\crefalias{subsection}{appsec}
\crefalias{subsubsection}{appsec}

In this section we prove \cref{thm:main_estimator} and \cref{thm:main_rates}. 

\section{Preliminaries}
\subsection{Notation}

We first introduce some relevant concepts from functional analysis. If $E$ is Hilbert space we denote by $\inner{.,.}_{E}$ and $\norm{.}_{E}$ its corresponding inner product and norm, respectively.
If $E $ and $F$ are two Hilbert spaces, we use $\Vert .\Vert $ to denote the operator norm  $ \Vert A \Vert  = \sup_{f: \Vert f \Vert \leq 1  } \Vert A f \Vert$, where $A$ is an operator from $E$ to $F$. We denote by $A^{*}$ the adjoint of $A$.

If $E$ is separable with an orthonormal basis $\{ e_k \}_k $, then $ \Vert . \Vert_1 $ and $ \Vert . \Vert_2 $ are the  trace norm and Hilbert-Schmidt norm on $E$ and are given by:
\begin{align}
	\Vert A\Vert_1 &= \sum_{k}\langle (A^{*}A)^{\frac{1}{2}}e_k,e_k \rangle\\
	\Vert A \Vert_2 &= \Vert A^{*}A \Vert_1    
.\end{align}
where $A$ is an operator from  $E$ to $E$. $ \lambda_{max}(A) $ is used to denote the algebraically largest eigenvalue of $A$.
 For $f$ in $E$  and $g$ in $F$ we denote by $ g \otimes f$ the tensor product viewed as an application from $E$ to $F$ with $ (g\otimes f)h = g \langle f,h \rangle_{E} $ for all $h$ in $E$. 
$C^1(\Omega )$ denotes the space of continuously differentiable functions on $\Omega$ and $ L^r(\Omega)$ the space of $r$-power Lebesgues-integrable function.
Finally for any vector $ \beta$ in $\R^{nd}$, we use the notation $ \beta_{(a,i)} = \beta_{(a-1)d + i} $ for $ a\in [n] $ and $ i\in [d] $. 

\subsection{Operator valued kernels and feature map derivatives}
Let $\x$ and $\y$ be two open subsets of $\R^p$ and $\R^d$ . $\h_{\y}$ is a reproducing kernel Hilbert space of functions $f: \y \rightarrow \R $ with kernel $k_{\y}$. We denote by  $\h$ a vector-valued reproducing kernel Hilbert space of functions $T: x\mapsto T_x$ from $\x$ to $\h_{\y}$ and we introduce the feature operator $\Gamma : x\mapsto \Gamma_x$ from $\x$ to $ \mathcal{L}(\h_{\y}, \h) $ where $\mathcal{L}(\h_{\y}, \h)$ is the set of bounded operators from $\h_{\y}$ to $\h$. For every $x\in \x$,  $\Gamma_x$ is an operator defined from $\h_{\y}$ to $\h$. 

The following reproducing properties will be used extensively:
\begin{itemize}
	\item Reproducing property of the derivatives of a function in $\h_{\y}$ ( \cite{Steinwart:2008a}, Lemma 4.34): provided that the kernel $k_{\y}$ is differentiable $m$-times with respect to each coordinate, then all $f \in \h_{\y} $ are differentiable for every multi-index $\alpha\in \n^d_0$ such that $\alpha \leq m$, and
	\begin{align}
		\partial^{\alpha} f(y) = \inner{f,  \partial^{\alpha} k(y, .)}_{\h_{\y}} && \forall y \in \y ,
	\end{align}
	where $ \partial^{\alpha} k_{y}(y, y^{'}) =  \frac{\partial^{\alpha} k(y,y^{'})}{\partial^{\alpha}y}$.
	In particular we will use the notation
	\begin{align}
		\partial_i k(y, y^{'}) =  \frac{\partial k(y,y^{'})}{\partial y_i}, && \partial_{i+d} k(y, y^{'}) =  \frac{\partial k(y,y^{'})}{\partial y^{'}_i} .
	\end{align}
	\item Reproducing property in the vector-valued space $\h$: For any $f\in \h_{\y}$ and any $T\in \h$ we have the following:
	\begin{align}
		\inner{T_x, f }_{\h_{\y}} = \inner{T, \Gamma_x f}_{\h}
	\end{align}
	In particular for every  $y\in \y$ we get:
	\begin{align}
		\inner{T_x, k(y,\cdot)}_{\h_{\y}} = \inner{T, \Gamma_x k(y,\cdot)}_{\h}
	\end{align}
	Using now the reproducing property in $\h_{\y}$ we get:
	\begin{align}
		T(x,y) := T_x(y) = \inner{T, \Gamma_x k(y,\cdot)}_{\h}
	\end{align}
\end{itemize}

\subsection{The conditional infinite dimensional exponential family}

Let $q_0$ be a base density function of a probability distribution over $\y$ and $\pi$ a probability distribution over $\x$. $\pi$ and $q_0$ are fixed  and are assumed to be supported in the whole spaces $\x$ and $\y$, respectively.

We introduce the following functions $Z : \h_{\y} \rightarrow \R^{*}_{+}$, such that for every $f \in \h_{\y}$  we have 
\begin{align}
	Z(f) := \int_{\y} \exp{ (\inner{f, k(y, .)}_{\h_{\y}} )} q_0(\ud y) .
\end{align}

We consider now the following family of operators
\begin{align}
\mathcal{T} = \{ T\in \h :  Z(T_x) < \infty , \forall x \in \x  \}
.\end{align}

This allows to introduce the Kernel Conditional Exponential Family as the family of conditional distributions satisfying 
\begin{align}
\mathcal{P} =\left\{  p_T(x|y) =  q_0(y) \frac{e^{(\langle T, \Gamma_x k(y,\cdot) \rangle_{\h}  )}}{Z(T_x)}\bigg\vert T \in \mathcal{T} \right\}
.\end{align}

Given samples $(X_i, Y_i)^{n}_{i=1}\in \x\times\y $ following a joint distribution $p_0$ the goal is to approximate the conditional density function $ p_0(y|x) $ in the case where $ p_0(y|x)\in \mathcal{P} $ (i.e. $ \exists T_0 \in \mathcal{T} $  such that $  p_0(y|x)= p_{T_0}(y|x) $ ).
To this end, we introduce the expected conditional score function between two conditional distributions $p(.|x)$ and $q(.|x)$ under $ \pi $,
\begin{align}
	J(p|| q) = \frac{1}{2} \int_x \int_y   \sum_{i=1}^d \big[\partial_i \log p(y|x) -\partial_i \log q(y|x) \big]^2 p(\ud y|x) \pi(\ud x) 
.\end{align}

This function has the nice property that $ J(p||q) \geq 0  $  and that $  J(p||q ) = 0  \Leftrightarrow q = p ,$ which makes it a good candidate as a loss function.

The marginal distribtion $p_0(x)$ doesn't have to match $\pi(x)$ in general as long as they have the same support.
For purpose of simplicity we will assume that $ p_0(x) = \pi(x) $.

\subsection{Assumptions}
\label{app:Assumptions}
	We make the following assumptions:
\begin{assumplist}
		\item \label{assumption_well_specified} (well specified) The true conditional density $p_0(y|x) = p_{T_0}(y|x)\in \mathcal{P} $ for some $T_0$ in $ \mathcal{T} $.
		\item \label{assumption_boundary}  $\y $ is a non-empty open subset of  of the form $\R^d $ with a piecewise smooth boundary $\partial\y:=\overline{\y}\setminus \y$, where $\overline{\y} $ denotes the closure of $\y$.
		\item  \label{assumption_kernel_smoothness} $k_{\y}$ is twice continuously differentiable on $\y\times \y$ and $\partial^{\alpha, \alpha} k_{\y}$ is continuously extensible to $\overline{\y}\times \overline{\y}$ for all $\vert \alpha \vert \leq 2  $.
		\item \label{assumption_vanishing_kernel} For all  $x\in \x$ and all  $i\in [d]$, as $y $ approaches $ \partial \y$ :
			$\Vert \partial_i k(y,\cdot) \Vert_{\y}p_0(y|x) = o( \norm{y}^{1-d} )$
		\item \label{assumption_gamma} The operator $\Gamma$ is uniformly bounded for the operator norm $\Vert  \Gamma_x  \Vert_{Op} \leq \kappa$ for all $x \in \x$.
		\item  \label{assumption_integrability}(Integrability) for some $\epsilon \geq 1$ and all $i\in [d]$:
		 \begin{gather}
			\Vert \partial_i k(y,\cdot) \Vert_{\y}  \in L^{2\epsilon}(\y, p_0 )
			,\
				 \Vert \partial^2_i k(y,\cdot)\Vert_{\y}  \in L^{\epsilon}(\y, p_0)
			,\
				\Vert \partial_i k(y,\cdot)\Vert_{\y} \partial_i \log q_0(y) \in L^{\epsilon}(\y, p_0)
		.\end{gather}
\end{assumplist}

\section{Theorems}

In this section, we prove the main theorems of the document, by extending the proofs of \cite{Sriperumbudur:2013} to the case of the vector-valued RKHS. We provide complete steps for all the proofs, including those that carry over from the earlier work, to make the presentation self-contained; the reader may compare with  \citep[Section 8]{Sriperumbudur:2013} to see the changes needed in the conditional setting.

\subsection{Score Matching}
\label{app:score_match}
\begin{theorem}[\textbf{Score Matching}]
    Under  \cref{assumption_well_specified,assumption_boundary,assumption_kernel_smoothness,assumption_vanishing_kernel,assumption_gamma,assumption_integrability}, the following holds:
    \begin{enumerate}
    	\item $J(p_{T_0}||p_T) < +\infty $ for all $T\in \mathcal{T}$
    	\item For all $T\in \mathcal{H} $ define
    	\begin{align}
    		J(T) = \frac{1}{2} \langle T - T_0, C(T-T_0) \rangle_{\h},
    		\label{eq:score_1}
    	\end{align} 
    	    where
    	    \begin{align}
    	    \label{eq:def_C}
    	  	 	    C:= \int_{\x\times \y} \underbrace{\sum_{i=1}^d  \big[\Gamma_{x}\partial_{i} k(y,\cdot)\otimes  \Gamma_{x}\partial_{i}k(y,\cdot) \big]}_{C_{x,y}} p_0(dx,dy)  = \E_{p_0}[C_{X,Y}]   	    	
    	    .\end{align}
  	  		
    	  then $C$ a trace-class positive operator on $\h$ and for all $T\in \mathcal{T}$ $J(T) = J(p_{T_0}||P_T)$.
\item Alternatively,
\begin{align}
	J(T) = \frac{1}{2} \langle T, CT\rangle_{\h} + \langle T ,\xi\rangle_{\h} + J(p_{T_0}||q_0)
.\end{align}    
    where 
    	\begin{align}
    	\label{eq:def_xi}
    		\h\ni \xi:= \int_{\x\times \y}  \underbrace{ \sum_{i=1}^d  \Gamma_x\big[ \partial_i \log q_0(y) \partial_i k(y,\cdot) + \partial_i^2 k(y,\cdot) \big] }_{\xi_{x,y}}  p_0(dx,dy) = \E_{p_0}[\xi_{X,Y}]
    	\end{align}
Moreover, $T_0$ satisfies $CT_0 = - \xi $
\item For any $\lambda>0 $, a unique minimizer $T_{\lambda} $ of $J_{\lambda}(T):= J(T) + \frac{\lambda}{2}\Vert T \Vert^2_{\h} $ over $\h$ exists and is given by:
 \begin{align}
	T_{\lambda} = -(C+ \lambda I )^{-1}\xi = (C+ \lambda I )^{-1}CT_0
.\end{align}
\end{enumerate}
\label{thm:score_match}
\end{theorem}

\begin{proof}
We prove the results in the same order as stated in the theorem: 
	\begin{enumerate}
	\item By the reproducing property of the real valued space $\h_{\y}$ we have: $T(x, y)  = \langle T_x, k(y,\cdot) \rangle_{\h_{\y}}$. Using the reproducing property for the derivatives of real valued functions in an RKHS in \cref{lemm:reprodcing_derivatives}, we get
 \begin{align}
 	\partial_i T(x, y)  = \partial_i\langle T_x, k(y,\cdot) \rangle_{\h_{\y}} = \langle T_x, \partial_ik(y,\cdot) \rangle_{\h_{\y}} && \forall i \in [d].
 \end{align}
 Finally, using  the reproducing property in the vector-valued space $\h$,
 \begin{align}
 	\partial_i T(x, y)= \langle T, \Gamma_{x} \partial_{i}k(y,\cdot) \rangle_{\h}, && \forall i \in [d].
 \end{align}
	it is easy to see that
	\begin{align}
	J(p_{T_0}||p_T)
	=
	\frac{1}{2}\int_{\x\times\y}  \sum_{i=1}^d \langle T_0 - T, \Gamma_x \partial_i k(y, .)\rangle_{\h}^2 p_0(\ud x,\ud y)
	\label{eq:score_expression_in_T}
	.\end{align}

		By \cref{assumption_gamma,assumption_integrability},
		 \begin{align}
		 	\Vert \Gamma_x \partial_i k(y,\cdot)\Vert_{\h} \leq \Vert \Gamma_x \Vert_{Op} \Vert\partial_{i} k(y,\cdot)  \Vert_{\h_{\y}} \leq \kappa  \sqrt{\partial_i \partial_{i+d} k(y,y)}  \in L^2(p_0)
		 ,\end{align}
and therefore by Cauchy-Schwarz inequality,
		 \begin{align}
		 	J(T) = J(p_{T_0}||p_T)  \leq  \frac{1}{2}\Vert T_0 - T\Vert_{\h}^2 \int_{\x\times\y}	\sum_{i=1}^d  \Vert  \Gamma_x \partial_i  k(y,\cdot) \Vert_{\h}^2 p_0(\ud x, \ud y)  < + \infty  
		 .\end{align}
 which means that  $J(T) <\infty $ for all $T \in \mathcal{T}$.
		 
		\item Starting from \cref{eq:score_expression_in_T}, it is easy to see that:
		 \begin{align}
	J(T)
	&=	
	\frac{1}{2} \int_{\x\times \y}	  \sum_{i=1}^d \langle T_0 - T, \Gamma_x\partial_i  k(y,\cdot)\otimes \Gamma_x\partial_i  k(y,\cdot) (T_0- T)  \rangle_{\h} p_0(\ud x,\ud y) \\
	&=
	\frac{1}{2} \int_{\x\times \y}	 \langle T_0 - T, C_{x,y}(T_0- T)  \rangle_{\h}   p_0(\ud x,\ud y) \\
	.\end{align}
	In the first line, we used the fact that $ \langle a , b \rangle_{\h}^2 = \langle a , b \rangle_{\h} \langle a , b \rangle_{\h} = \langle a , b\otimes b a \rangle_{\h} $ for any $a$ and $b$ in a Hilbert space $\h$.
	By further observing that $C_{x,y}$ and $  (T_0-T)\otimes (T_0 - T ) $ are Hilbert-Schmidt operators as $ \Vert C_{x,y} \Vert_{HS} \leq \kappa^2 \sum_{i=1}^d \Vert \partial_i k(y,\cdot) \Vert < \infty $ by \cref{lemm:C_xy_Trace_class} and $\Vert (T_0-T)\otimes (T_0 - T )\Vert_{HS} = \Vert(T_0-T)\Vert^2_{\h}  < \infty$ we get that:
	\begin{align}
	J(T) = 
	\frac{1}{2} \int_{\x\times \y}   \langle (T_0 - T)\otimes (T_0-T), C_{x,y}  \rangle_{HS} p_0(\ud x,\ud y)\\
	\end{align}
	Using  \cref{assumption_integrability} we have by \cref{lemm:bochner_integrability} that $C_{x,y}$ is $p_0$-integrable in the Bochner sense (see \cite{Retherford:1978}) Definition 1)  and that the inner product and  integration may be interchanged:
	\begin{align}
	J(T) &= 
	\frac{1}{2}   \Bigg\langle (T_0 - T)\otimes (T_0-T),\int_{\x}\int_{\y}	 C_{x,y} p_0(\ud x,\ud y)  \Bigg\rangle_{HS} = 
	  \frac{1}{2} \langle T_0-T, C(T_0-T) \rangle_{\h}
	\end{align} 
	\item
	From \cref{eq:score_1} we have $J(T) =  \frac{1}{2} \langle T, CT \rangle_{\h} - \langle T, CT_0 \rangle_{\h} + \frac{1}{2} \langle T_0, CT_0 \rangle_{\h} $.
	Recalling that: $\partial_{i} T(x,y) = \langle T, \Gamma_x \partial_{i} k(y,\cdot) \rangle_{\h} $ for all $ i\in [d] $, and using $\partial_{i} T_0(x,y) = \partial_{i} \log p_0(y|x) - \partial_{i}\log q_0(y|x) $ one gets:
	\begin{align}
	\langle T, CT_0 \rangle_{\h}
	=&
	\int_{\x\times \y} \Big[ \sum_{i=1}^d\partial_i T(x,y)\partial_i T_0(x,y)\Big]p_0(\ud x,\ud y) \\
	=&
	\int_{\x\times \y} \Big[ \sum_{i=1}^d\partial_i T(x,y)\partial_{i} \log p_0(y|x)\Big]  p_0(\ud x )\ud y
	- \int_{\x\times \y}\Big[ \sum_{i=1}^d\partial_i T(x,y)\partial_{i} \log q_0(y|x)\Big] p_0(\ud x,\ud y) \\
	\stackrel{(a)}{=}&
	\int_{\x} p_0(\ud x)\int_{\partial \y} p_0(y|x) \nabla_{y}T(x,y).\vec{dS} 
	- \int_{\x\times \y}  \Big[ \sum_{i=1}^d \partial_i^2T(x,y) + \partial_i T(x,y)\partial_{i} \log q_0(y|x) \Big]p_0(\ud x,\ud y) 
	.\end{align}
$(a)$ is obtained using the first Green's identity, where $\partial \y$ is the boundary of $\y$ and $\vec{dS}$ is the oriented surface element. The first term $\int_{\x} \pi(\ud x) \int_{\partial \y} p_0(y|x) \nabla_{y}T(x,y).\vec{dS}$ vanishes by \cref{lemm:vanishing_surface_integral}, which relies on \cref{assumption_vanishing_kernel}. The second term can be written as: 
	$\int_{\x\times \y}  \langle T, \xi_{x,y}\rangle_{\h} p_0(\ud x,\ud y)$.

By  \cref{assumption_integrability,assumption_gamma} $\xi_{x,y}$ is Bochner $p_0$-integrable, therefore:
\begin{align}
	\int_{\x\times \y}  \langle T, \xi_{x,y}\rangle_{\h} p_0(\ud x,\ud y)=\Big\langle T,  \int_{\x\times \y}  \xi_{x,y} p_0(\ud x,\ud y) \Big\rangle_{\h} = \langle T, \xi \rangle_{\h}
.\end{align}
Hence $ \langle T, CT_0 \rangle_{\h}  = - \langle T, \xi \rangle_{\h}   $ and $\xi = - CT_0$. 
Moreover, one can clearly see that:
\begin{align}
	 \langle T_0, CT_0 \rangle_{\h} 
		&=
	   \int_{\x\times\y}  \sum_{i=1}^d(\partial_i T_0(x,y))^2 p_0(\ud x,\ud y) =
	   J(p_{T_0} || q_0) 
.\end{align}
And the result follows.

\item For $\lambda>0$, $(C+\lambda I)$ is invertible as $C$ is a symmetric trace-class operator. Moreover, $ (C+\lambda I)^{\frac{1}{2}} $ is well defined and one can easily see that:  
\label{proof:solution}
\begin{align}
J_{\lambda}(T) = \frac{1}{2}\Vert (C+\lambda I )^{\frac{1}{2}} T + (C+\lambda I )^{-\frac{1}{2}} \xi \Vert^2_{\h} - \frac{1}{2} \langle\xi, (C+\lambda I )^{-1} \xi \rangle_{\h} + c_0	
\end{align}
with $c_0 = J(p_{T_0}||q_0)$.
$J_{\lambda}(T)$ is minimized if and only if $ (C+\lambda I )^{\frac{1}{2}} T = (C+\lambda I )^{-\frac{1}{2}} \xi$ and therefore $T =  (C+\lambda I )^{-1}\xi$ is the unique minimizer of $J_{\lambda}(T)$. 

\end{enumerate}
\end{proof}

\subsection{Estimator of $T_0$}
\label{app:Estimator}
Given samples $(X_a,Y_a)_{a=1}^n$ drawn i.i.d. from $p_0$ and $\lambda>0$, we define the empirical score function as 
 \[
 \hat{J}(T) := \frac{1}{2} \langle T, \hat{C} T \rangle_{\h} + \langle T, \hat{\xi} \rangle_{\h}  + J(p_{T_0}||q_0) 
 \label{eq:empirical_score}
.\] 
where:
\begin{align}
	\hat{C} &:=  \frac{1}{n} \sum_{a=1}^n \sum_{i=1}^d \Gamma_{X_a} \partial_i k(Y_a,\cdot) \otimes \Gamma_{X_a} \partial_i k(Y_a,\cdot)\\
	\hat{\xi} :&=  \frac{1}{n} \sum_{a=1}^n \sum_{i=1}^d \Gamma_{X_a} \big[ \partial_i \log q_0(Y_a) \partial_i k(Y_a,\cdot), + \partial_i^2 k(Y_a,\cdot) \big]
.\end{align}
are the empirical estimators of $C$ and $\xi$ respectively.

\begin{theorem}[\textbf{Estimator of $T_0$}]
	\label{thm:estimator_T} 
	For and any $\lambda>0$, we have the following:
	\begin{enumerate}

	  \item 
 
	  The unique minimizer $T_{\lambda,n}$ of $ \hat{J}_{\lambda}(T):= \hat{J}(T) + \frac{\lambda}{2}\Vert T \Vert^2_{\h}$ over $\h $ exists and is given by
 \begin{align}
 	T_{\lambda, n} = - (\hat{C} + \lambda I )^{-1}\hat{\xi}
 .\end{align}

\item Moreover, $T_{\lambda, n}$ is of the form
\begin{align}
T_{\lambda, n} = -\frac{1}{\lambda} \hat{\xi} + \sum_{b=1}^n\sum_{i=1}^d \beta_{(b-1)d + i}\Gamma_{X_b}\partial_i k(Y_b,\cdot) 	,
\end{align}
where $(\beta_b)$ are obtained by solving the following linear system:
\begin{align}
(G + n \lambda I )\beta = \frac{h}{\lambda}
\label{eq:system}
\end{align}
with:
\begin{align}
	(G)_{(a-1)d + i, (b-1)d +j} = \langle \Gamma_{X_a} \partial_i k(Y_a, .), \Gamma_{X_b} \partial_j k(Y_b, .)\ \rangle_{\h}
.\end{align}
and:
\begin{align}
(h)_{(a-1)d + i} = \langle \hat{\xi}, \Gamma_{X_a} \partial_i k(Y_a, .) \rangle_{\h}	
.\end{align}

	\end{enumerate}
\end{theorem}

\begin{proof}

\begin{enumerate}
	\item The same proof as in \cref{thm:score_match} holds with $C$ and $\xi$ replaced by $\hat{C} $ and $\hat{\xi}$.

\item We will use the general representer theorem stated in \cref{lemm:representer}.  We have that:
\begin{align}
	T_{\lambda, n} 
	&=
	 \arginf_{T\in \h} \frac{1}{2}\langle T \hat{C}T  \rangle_{\h} + \langle T, \hat{\xi} \rangle_{\h} + \frac{\lambda}{2}\Vert T\Vert^2_{\h}\\
	 &=
	 \arginf_{T \in \h} \frac{1}{2} \sum_{a=1}^n \sum_{i=1}^d \langle T, \Gamma_{X_a}\partial_{i}k(Y_a, .) \rangle_{\h}^2    + \langle T, \hat{\xi} \rangle_{\h}  + \frac{\lambda}{2}\Vert  T \Vert^2_{\h}\\
	 &=
	 \arginf_{T\in \h} V(\langle T, \phi_1 \rangle_{\h}, ...,\langle T, \phi_{nd+1} \rangle_{\h}   ) + \frac{\lambda}{2}\Vert T \Vert^2_{\h}
.\end{align}

Where $V(\theta_1, ..., \theta_{nd+1}):= \frac{1}{2n} \sum_{a=1}^n\sum_{i=1}^d \theta^2_{(a-1)d+i} + \theta_{nd+1}  $ is a convex differentiable function and $\phi_{(a-1)d+i}:= \Gamma_{X_a}\partial_{i} k(Y_a, .) $ where $a\in [n], i\in [d]$ and $\phi_{nd+1} = \hat{\xi}$. Therefore, it follows from \cref{lemm:representer} that:
\begin{align}
T_{\lambda, n}
&=
\delta \hat{\xi}  + \sum_{a=1}^n \sum_{i=1}^d \beta_{(a-1)d+i} \phi_{(a-1)d +i}	
.\end{align}
where $\delta$ and $\beta$ satisfy:
\begin{align}
	\lambda(\beta, \delta ) + \nabla V(K(\beta,\delta)) = 0
\end{align}
with $K = \begin{pmatrix}
	G & h\\
	h^T & \Vert \hat{\xi} \Vert_{\h}^2
\end{pmatrix}$.

The gradient $\nabla V$ of $V$ is given by $\nabla V (z,t) = (\frac{1}{n} z, 1 )$.  The above equation reduces then to $\lambda\delta +1=0$ and $\lambda \beta  + \frac{1}{n}G \beta + \frac{\delta}{n} h =0$ which yields $\delta = -\frac{1}{\lambda}$ and  $(\frac{1}{n}G + \lambda I  )\beta = \frac{1}{n\lambda } h$.
\end{enumerate}
\end{proof}

\subsection{Consistency and convergence}
\label{app:consistency}
\begin{theorem}[\textbf{Consistency and convergence rates for $T_{\lambda, n}$}]
	\label{thm:rates_appendix}
Let $\gamma > 0$ be a positive number and define $\alpha = \max(\frac{1}{2(\gamma + 1)}, \frac{1}{4})\in (\frac{1}{4}, \frac{1}{2} ) $, 	under \cref{assumption_well_specified,assumption_boundary,assumption_kernel_smoothness,assumption_vanishing_kernel,assumption_gamma,assumption_integrability}:
	\begin{enumerate}
		\item if $T_0 \in \overline{\mathcal{R}(C)} $ then $\Vert T_{n,\lambda} - T_0 \Vert \rightarrow 0 $  when $\lambda \sqrt{n} \rightarrow \infty $, $\lambda \rightarrow 0$ and $ n\rightarrow \infty $.
		\item if $T_0 \in \mathcal{R}(C^{\gamma}) $ for some $\gamma >0 $ then $\Vert T_{n,\lambda} - T_0 \Vert = \mathcal{O}_{p_0}( n^{-\frac{1}{2} + \alpha}) $  for $ \lambda = n^{-\alpha} $
	\end{enumerate}
\end{theorem}

\begin{proof}
 Recalling that $ T_{\lambda,n}=  -(\hat{C} + \lambda I)^{-1}\hat{\xi}  $
We consider the following decomposition:
\begin{align}
T_{\lambda, n} - T_{\lambda}
	&=
-(\hat{C} + \lambda I)^{-1}(\hat{\xi} + (\hat{C} + \lambda I )T_{\lambda}  )
	 \stackrel{(*)}{=} -( \hat{C} + \lambda I )^{-1}(\hat{\xi} + \hat{C}T_{\lambda}  + C( T_0 - T_{\lambda} ) )\\
	 &= (\hat{C } + \lambda I   )^{-1}(C - \hat{C})(T_{\lambda} - T_0) - (\hat{C} + \lambda I )^{-1}( \hat{\xi} + \hat{C} T_0 )\\
	 &= (\hat{C}  + \lambda I )^{-1}(C - \hat{C})(T_{\lambda} - T_0 ) - (\hat{C} + \lambda I  )^{-1}(\hat{\xi} - \xi) + ( \hat{C} + \lambda)^{-1}(C - \hat{C})T_0   
.\end{align}
	We used the fact that $\lambda T_{\lambda} = C(T_0 - T_{\lambda})$ in $(*)$. Define now
	
	\begin{align}
		S_1 &:= \Vert (\hat{C} + \lambda I)^{—1}(C - \hat{C})(T_{\lambda} - T_0) \Vert_{\h}\\
		S_2 &:= \Vert (\hat{C} + \lambda I )^{-1}(\hat{\xi} - \xi ) \Vert_{\h}\\
		S_3 &:=  \Vert (\hat{C} + \lambda I )^{-1}(C - \hat{C} )T_0 \Vert_{\h}\\
		\mathcal{A}_0(\lambda) &:= \Vert T_{\lambda, n} - T_0 \Vert_{\h}
	.\end{align}
it comes then:
	\begin{align}
	\Vert T_{\lambda} - T_0 \Vert_{\h}
	&\leq 
		\Vert T_{\lambda, n} - T_{\lambda} \Vert_{\h}
+\Vert T_{\lambda} - T_0 \Vert_{\h}\\
	&\leq
	S_1 + S_2 + S_2 + \mathcal{A}_0(\lambda) ,
	\end{align}

Using \cref{thm:rates} we can bound $S_1$, $S_2$ and $S_3$. Note that $C_{x,y}$ as defined in \cref{eq:def_C} is a positive, self-adjoint trace-class operator by \cref{lemm:C_trace_class} , we therefore have:
	\begin{align}
		 \Vert C_{x,y}\Vert_{HS}^2
		&=
		\sum_{i,j=1}^d  \langle \Gamma_{x}\partial_i k(y,\cdot), \Gamma_{x}\partial_j k(y,\cdot) \rangle_{\h}^2 \leq
		\sum_{i,j=1}^d \Vert\Gamma_x\partial_i k(y,\cdot) \Vert^2_{\h}\Vert\Gamma_x\partial_j k(y,\cdot) \Vert^2_{\h}\\
		&\leq
		(\sum_{i=1}^d \Vert \Gamma_x \partial_{i} k(y,\cdot) \Vert^2_{\h}  )^2 \leq 
		d\sum_{i=1}^d \Vert \Gamma_x \partial_i k(y,\cdot) \Vert_{\h}^4 \leq 
		d\kappa^4 \sum_{i=1}^d \Vert \partial_i k(y,\cdot) \Vert^4_{\h_{\y}}
	.\end{align}   	
	The last inequality is obtained using  \cref{assumption_gamma}.  Using now \cref{assumption_integrability}  for $\epsilon = 2$ 
	one can get:
\begin{align}
\int_{\x\times \y} \Vert C_{x,y} \Vert^2_{HS}p_0(\ud x,\ud y)	
&\leq
d \kappa^4 \sum_{i=1}^d \int_{\x\times \y}\Vert\partial_{i}k(y,\cdot) \Vert^4_{\h_{\y}} p_0(\ud x , \ud y) < +\infty
.\end{align}
	\cref{thm:rates} can then be applied to get the following inequalities: 
		\begin{align}
		S_1 \leq  &\Vert (\hat{C} + \lambda I)^{-1} \Vert \Vert  (C- \hat{C} )(T_{\lambda} - T_0 )\Vert_{\h} = \mathcal{O}_{p_0}( \frac{\mathcal{A}(\lambda) }{\lambda \sqrt{n}})\\
		  S_3 \leq &\Vert (\hat{C} + \lambda I)^{-1} \Vert \Vert (C-\hat{C})T_0 \Vert   =\mathcal{O}_{p_0}(\frac{1}{\lambda \sqrt{n}} ) \Vert_{\h} \\
		  &\Vert (C+\lambda I)^{-1} \Vert \leq \frac{1}{\lambda}	
	\end{align}
To bound $S_2$ we need to show that $ \Vert \hat{\xi} - \xi \Vert_{\h} = \mathcal{O}_{p_0}(n^{-\frac{1}{2}}) $. The same argument as in \cite{Sriperumbudur:2013} holds:

\begin{align}
	\E_{p_0}\Vert \hat{\xi} - \xi\Vert^2_{\h} 
	&= \frac{1}{n}\Bigg( \int_{\x\times \y} \Vert \xi_{x,y} \Vert^2_{\h} p_0(\ud x,\ud y)  - \Vert \xi\Vert^2 \Bigg)\\
	&\leq \frac{1}{n}\int_{\x\times \y} \Vert \xi_{x,y} \Vert^2_{\h} p_0(\ud x,\ud y)
\end{align}

By  \cref{assumption_integrability} for $\epsilon = 2$ we have that $\int_{\x\times \y} \Vert \xi_{x,y} \Vert^2_{\h} p_0(\ud x,\ud y)
 < \infty$. One can therefore apply Chebychev inequality to get the results. It comes that:
	\begin{align}
		S_2 \leq \Vert(\hat{C} + \lambda I )^{-1}  \Vert \Vert \hat{\xi} - \xi  \Vert_{\h} = \mathcal{O}_{p_0}(\frac{1}{\lambda \sqrt{n}}) 
	\end{align}
Using the bounds on $S_1$, $S_2$ and $S_3$ we get:
\begin{align}
\Vert T_{\lambda, n} - T_0\Vert_{\h} = \mathcal{O}_{p_0}(\frac{1}{\lambda \sqrt{n}}  + \frac{\mathcal{A}_{0}(\lambda)}{\lambda \sqrt{n}} ) + 	\mathcal{A}_{0}(\lambda)
\label{eq:useful_bound}
\end{align}

\begin{enumerate}
	\item  By \cref{lemm:range_space} we have $ \mathcal{A}_{0}(\lambda) \rightarrow 0  $ as  $\lambda \rightarrow 0 $ if $ T_0\in \overline{\mathcal{R}(C)} $. Therefore it follows from \cref{eq:useful_bound} that $ \norm{ T_{\lambda, n}  - T_0} \rightarrow 0  $ as $ \lambda \rightarrow 0$, $\lambda\sqrt{n} \rightarrow \infty $ and $n\rightarrow \infty$.

\item  We have by \cref{lemm:range_space} that  if $T_0 \in \mathcal{R}(C^{\gamma}) $ for $\gamma > 0$ then:
\begin{align}
 \mathcal{A}_0(\lambda) \leq \max \{1, \Vert  C \Vert^{\gamma-1}  \}\Vert C^{-\gamma} T_0\Vert_{\h} \lambda^{\min\{1, \gamma\} }	
.\end{align}
The result follows by choosing $ \lambda = n^{- \max\{\frac{1}{4}, \frac{1}{2(\gamma+1)} \} }  = n^{-\alpha}$.
\end{enumerate}

\end{proof}

We denote by  $ KL(p_{T_0}||p_{T}) $ the expected $KL$ divergence between $p_{T_0}$ and $p_T$ under the marginal $p_0(x)$.

\begin{theorem}[\textbf{Consistency and convergence rates for $p_{T_{\lambda, n}}$} ]
	Assuming \cref{assumption_well_specified,assumption_boundary,assumption_kernel_smoothness,assumption_vanishing_kernel,assumption_gamma,assumption_integrability}, and  $ \norm{k}_{\infty}:=  \sup_{y\in\y} k(y,y)  < \infty $ and that $p_{T_0}(y|x)$ is supported on $\y$ for all $x\in\x$  then the following holds:

		\begin{enumerate}
		\item $ KL(p_{T_0}||p_{T_{\lambda, n}})\rightarrow 0 $ as $\lambda \sqrt{n} \rightarrow \infty $, $\lambda \rightarrow 0$ and $ n\rightarrow \infty $. 
		\item If $T_0 \in \mathcal{R}(C^{\gamma}) $ for some $\gamma >0 $ then by defining $\alpha = \max(\frac{1}{2(\gamma + 1)}, \frac{1}{4})\in (\frac{1}{4}, \frac{1}{2} ) $, and choosing $ \lambda = n^{-\alpha} $ we have that $KL(p_0|| p_{T_{n,\lambda}}) = \mathcal{O}_{p_0}( n^{-1 + 2\alpha} )$
	\end{enumerate} 
	\label{thm:KL_rates}
 \end{theorem}

\begin{proof}
By \cref{lemm:bounded_kernel}, we have that $ \mathcal{T} = \h $  and we can assume without loss of generality that  $T_0  \in \overline{\mathcal{R}(C)}$. 
Using \cref{lemm:bound_KL} (also see  \cite{Vaart:2008} Lemma 3.1 ), one can see that for a given $x$ :
\begin{align}
	KL(p_{T_0}(Y|x)|| p_{T_{\lambda, n}}(Y|x)) \leq  \norm{T_0(x) - T_{\lambda, n}(x)}_{\infty}^2 \exp{\norm{T_0(x) - T_{\lambda, n}(x)}_{\infty}} (1 +  \norm{T_0(x) - T_{\lambda, n}(x)}_{\infty})
\label{eq:KL_bound}
\end{align}
Moreover, using \cref{assumption_gamma} and the fact that  $ \norm{k}_{\infty} < \infty $ one can see that 

\begin{align}
	\vert T_0(x,y) - T_{\lambda, n}(x,y) \vert_{\h_{\y}} 
	&=
	 \inner{T_0- T_{\lambda, n}, \Gamma_x k(y,\cdot) }_{\h}\\
	 &\leq
	  \norm{T_0 - T_{\lambda, n}}_{\h}\norm{\Gamma_x k(y,\cdot) }_{\h}
\end{align}   
which gives after taking the supremum:
\begin{align}
\norm{T_0(x) - T_{\lambda, n}(x)}_{\infty}  \leq  \kappa \norm{k}_{\infty}\norm{T_0 - T_{\lambda, n}}_{\h}
\label{eq:bound_sup_norm}
\end{align}

for all $x \in \x$. Using \cref{eq:bound_sup_norm} in \cref{eq:KL_bound} and taking the expectation with respect to $ x$, one can conclude using \cref{thm:rates_appendix}.

\end{proof}

\section{Auxiliary results}
\label{app:auxiliary}
\begin{lemma} 
	\label{lemm:C_xy_Trace_class}
Under \cref{assumption_kernel_smoothness,assumption_gamma,assumption_integrability} we have that:
	\label{lemm:C_trace_class}
	\begin{enumerate}
		\item $C_{x,y}$ is a trace-class positive and symmetric operator for all $(x,y) \in \x\times\y $
		\item $C_{x,y}$ is Bochner-integrable for all $(x,y) \in \x\times\y $
		\item $C$ is a trace-class positive and symmetric operator
	\end{enumerate}

\end{lemma}

\begin{proof}
	Recall that $C = \int_{\x\times\y} C_{x,y} p_0(\ud x,\ud y) $ 
		where $C_{x,y} = \sum_{i=1}^d \Gamma_x \partial_{i} k(y,\cdot) \otimes\Gamma_x \partial_{i} k(y,\cdot) $ is a positive self-adjoint operator. The trace norm of $C_{x,y}$ satisfies: 
		\begin{align}
			\Vert C_{x,y}\Vert_1 
			&\leq
			 \sum_{i=1}^d \Vert \Gamma_x \partial_{i} k(y,\cdot) \otimes \Gamma_x \partial_{i} k(y,\cdot)\Vert_1\\
			 &=
			 \sum_{i=1}^d \Vert \Gamma_x \partial_{i} k(y,\cdot) \Vert_{\h}^2
			 \leq
			 \sum_{i=1}^d \Vert \Gamma_x\Vert_{Op}^2 \Vert \partial_{i} k(y,\cdot) \Vert_{\h_{\y}}^2\\
			 &\stackrel{(a)}{\leq}
			 \kappa^2 \sum_{i=1}^d \Vert \partial_{i} k(y,\cdot) \Vert_{\h_{\y}}^2 < \infty
		.\end{align}
		
		$(a)$ comes from \cref{assumption_gamma}. This implies that $C_{x,y}$ is trace-class. Moreover, by   \cref{assumption_integrability} for $\epsilon = 1$ : $\Vert \partial_{i} k(y,\cdot) \Vert_{\h_{\y}} \in  L^{2\epsilon}(\y, p_0)$ which leads to:
		\begin{align}
			\int_{\x\times \y}\Vert C_{x,y}\Vert_1 p_0(\ud x , \ud y) < \infty 
		.\end{align}
		 This means that $C_{x,y}$ is $p_0$-integrable in the Bochner sense ( \cite{Retherford:1978}, Definition 1 and Theorem 2   ) and its integral $C$ is trace-class with: 

		  \begin{align}
		  	\Vert C\Vert_1 
		  	&=
		  	\Big\Vert\int_{\x\times\y} C_{x,y} p_0(\ud x , \ud y) \Big\Vert_1
		  	\leq 
		  	\int_{\x\times \y } \Vert C_{x,y}\Vert_1 p_0(\ud x , \ud y) < \infty
		  .\end{align}	
\end{proof}

\begin{lemma}
\label{lemm:bochner_integrability}
	Let $\x$ be a topological space endowed with  a probability distribution $\PP$. Let $B$ be a separable Banach space. 
	Define $R$ to be an  $B$-valued measurable function on $\x$ in the Bochner sense ( \cite{Retherford:1978} Definition 1 ),   satisfying $ \int_{\x} \Vert  R(x) \Vert_{B} d\mathbb{P}(x) < \infty $, then $R$ is $\PP$-integrable in the Bochner sense (\cite{Retherford:1978} Definition 1, Theorem 6) and for any continuous linear operator $T$ from  $B$ to another Banach space $A$, then $ T R $ is also   $\PP$-integrable in the Bochner sense and:
	\begin{align}
		\int_{\x} TR(x) d\PP(x) = T\int_{\x} R(x) d\PP(x)
	\end{align}	
\end{lemma}
For a proof of this result see \cite{Retherford:1978}, Definition 1, Therorem 6 and 7.

\begin{lemma}[\textbf{RKHS of differentiable kernels (\cite{Steinwart:2008a} Chap 4.4, Corollary 4.36)}]
	\label{lemm:reprodcing_derivatives}
Let $\x\in \R^d $ be an open subset, $m\geq 0$, and $k$ be an $m$-times continuously differentiable kernel on $\x$ with RKHS $\h$. Then every function $f\in\h$ is $m$-times continuously differentiable, and for $\alpha \in  \n^d_0 $ with $\vert \alpha \vert \leq m$ we have:
\begin{align}
	\vert \partial^{\alpha}f(x) \vert 
	&\leq
	\norm{f}_{\h}^2 (\partial^{\alpha, \alpha} k(x,x))^{\frac{1}{2}}\\ 
	\partial^{\alpha}f(x) 
	&=
	\inner{f, \partial^{\alpha}k(x,\cdot)}_{\h} 
\end{align}

\end{lemma}
A proof of this result can be found in \cite{Steinwart:2008a} (Chap 4.4, Corollary 4.36)

\begin{lemma}
	\label{lemm:vanishing_surface_integral}
Under \cref{assumption_boundary,assumption_vanishing_kernel,assumption_kernel_smoothness} we have the following:
\begin{align}
	\int_{\x} \pi(\ud x)\int_{\partial \y} p_0(y|x) \nabla_{y}T(x,y).\vec{dS} = 0 && \forall T \in \mathcal{T}
\end{align}
where $\partial \y$ is the boundary of $\y$ and $\vec{dS}$ is an oriented surface element of $\partial \y$.

\end{lemma}
\begin{proof}
First let's  prove that $ \Vert \nabla_y T(x, y)  \Vert p_0(y|x)  = o(\Vert y \Vert^{1-d} ) $ for all $x\in \x$. Where the norm used is the euclidian norm in $\R^d$. Using the reproducing property and Cauchy-Schwarz inequality one can see that:
\begin{align}
	\Vert  \nabla_y T(x, y) \Vert^2 
	&=
 	\sum_{i=1}^d (\partial_i T(x, y))^2 = \sum_{i=1}^d \inner{T_x, \partial_i k(y, .)}^2\\
 	&\leq 
 	  \norm{T_x}^2\Big(\sum_{i=1}^d \norm{\partial_i k(y, .)}^2\Big)  
\end{align}
By \cref{assumption_vanishing_kernel}, one can see that $ \sqrt{\sum_{i=1}^d \norm{\partial_i k(y, .)}^2}p_0(y|x) = o(\norm{x}^{1-d})  $, therefore it comes that $ \Vert \nabla_y T(x, y)  \Vert p_0(y|x)  = o(\Vert y \Vert^{1-d} ) $.
Using \cref{lemm:vanishing_surface_integral_0} one gets that $\int_{\partial \y} p_0(y|x) \nabla_{y}T(x,y).\vec{dS} = 0 $ for all $x \in \x$ which leads to the result.

\end{proof}

\begin{lemma}
	\label{lemm:vanishing_surface_integral_0}
Let $\Omega$ be an open set in $\R^d$ with piece-wise smooth boundary $\partial \Omega$. Let $u$ be a real valued function defined over $\Omega$ and $v : \R^d \rightarrow \R^d$ a vector valued function. We assume that $ u $ and $v$ are measurable and  that $ \Vert v(x) \Vert \vert u(x) \vert  = o( \Vert x\Vert^{1-d}  ) $. Then the following surface integral is null:
\begin{align}
	\int_{\partial \Omega} u(x)v(x).\vec{\ud S} = 0
\end{align} 
where $\vec{\ud S}$ is an element of the surface $\partial \Omega$.
\end{lemma}
More details on this result can be found in \cite{Pietzsch:1994}

\begin{lemma}[Generalized representer theorem]
	\label{lemm:representer}
Let $\h$ be a vector-valued Hilbert space and let $(\phi_i)^m_{i=1} \in  \h^m $. Suppose $J: \h \rightarrow \R$ is such that $ J(T) = V( \langle T, \phi_1 \rangle_{\h}, ...,  \langle T, \phi_m \rangle_{\h} ) $ for $ T\in \h $, where $ V:\R^m \rightarrow \R $ is a convex and  g\^{a}teaux-differentiable function. Define:
\[ T_{\lambda} = \arginf_{T\in\h} J(T) + \frac{\lambda}{2} \Vert T \Vert^2_{\h} 
\]
where $\lambda>0$. Then there exists $(\alpha_i)^{m}_{i=1} \in \R^m $ such that $ T_{\lambda} = \sum_{i=1}^m  \alpha_i \phi_i  $ where $ \alpha := (\alpha_1, ..., \alpha_m) $ satisfies the following equation:
\[ (\lambda I  + (\nabla V  ) \circ K)\alpha = 0
,\]
with $ (K)_{i,j} =  \langle\phi_i, \phi_j\rangle_{\h}, \i \in [m], j\in [m]  $

\end{lemma}

\begin{proof}
	Define $A:\h \rightarrow \R^m $, $T \mapsto (\langle  T, \phi_{i} \rangle_{\h} )^m_{i=1} $. Then $ T_{\lambda} = \arginf_{T\in\h} V(AT) + \frac{\lambda}{2} \Vert T \Vert^2_{\h} $. Taking the g\^{a}teaux-differential at $T$,  the optimality condition yields:
	\begin{align}
		0 = A^{*} \nabla V(AT_{\lambda}) + \lambda T_{\lambda} 
		&\Leftrightarrow 
		A^{*}\Big( -\frac{1}{\lambda} \nabla V(AT_{\lambda}) \Big) = T_{\lambda}\\
		&\Leftrightarrow 
		(\exists \alpha \in \R^m ) T_{\lambda} = A^{*}\alpha ,  \alpha = -\frac{1}{\lambda} \nabla V(AT_{\lambda})\\
		&\Leftrightarrow
		(\exists \alpha \in \R^m ) T_{\lambda} = A^{*}\alpha ,  \alpha = -\frac{1}{\lambda} \nabla V(AA^{*}\alpha)\\
	\end{align}
	where $A^{*}:\R^m \rightarrow \h $ is the adjoint of $A$ which can be obtained as follows. Note that: 
	\begin{align}
		(\forall T \in \h  )\text{ } (\forall \alpha \in \R^m ) && \langle AT,\alpha\rangle = \sum_{i=1}^m \alpha_i \langle T, \phi \rangle_{\h} = \Big\langle T, \sum_{i=1}^m \alpha_i \phi_i \Big\rangle_{\h}
	\end{align} 	
	thus $A^{*}\alpha = \sum_{i=1}^m \alpha_i\phi_i$. Therefore $AA^{*}\alpha = \sum_{i=1}^m \alpha_j A\phi_j  = \sum_{j=1}^m \alpha_j (\langle\phi_j, \phi_i  \rangle_{\h} ) $ and hence $AA^{*} = K$.
\end{proof}

\begin{lemma}[Bound on KL divergence between $p_f$ and $p_g$   ( \cite{Vaart:2008} Lemma 3.1 )]
	\label{lemm:bound_KL}
	Assume that $\norm{k}_{\infty} < \infty $ and let $f$ and $g$ in $\h_{\y}$ such that $ Z(f)  $ and $Z(g)$ are finite,  then:
	$KL(p_{f}||q_{g}) \leq  \norm{f- g}_{\infty}^2 \exp{ \norm{f- g}_{\infty}} (1 +  \norm{f- g}_{\infty} )$	
\end{lemma}

\begin{lemma}[see Lemma 14 in \cite{Sriperumbudur:2013}]
	\label{lemm:bounded_kernel}
Suppose $ \sup_{y\in\y} k(y,y) < \infty $ and $ supp(q_0) = \y $. Then $\mathcal{T} = \h $ and for any $T_0$ there exists $ \widetilde{T}_O \in \overline{\mathcal{R}(C)} $ such that $ p_{\widetilde{T}_0} = p_0 $.
\end{lemma}

\begin{proof}
Since $ \norm{k}_{\infty}  < \infty $ then $ Z(T_x) \leq  \exp{\norm{T_x} \norm{k}_{\infty} } <  \infty$ for all $ T \in \h $, therefore  $ \mathcal{T} = \h $.
Moreover, since $ \supp (p_{T_0})(y|x)  = \y$ for all $x$ in  $ \x $, this implies that the null space of $C$  $\mathcal{N}(C)$ can either be the set of functions $ T(x,y) = m(x) $ or $\{0\}$. Indeed, for  $T \in \mathcal{N}(C) $ we have $  \inner{T, CT} = 0 $ which leads to $  \int_{\x\times\y}   \norm{\nabla_{y} T}^2_2 p_0(\ud x, \ud y)  = 0 $ which means that $p_0$-almost surely, $  T_x(y) = m(x) $ a constant function of $y$ if the set of constant functions belong to $\h_{\y}$, or $T_x(y) = 0$ otherwise.
Let $ \widetilde{T_0}  $ be the orthogonal projection of $T_0$ onto $ \overline{ \mathcal{ R }(C)} = \mathcal{N}(C)^{\perp}$ then  $T_0$ can be written in the form $ T_0(x,y) = m(x) + \widetilde{T}_0(x,y) $. It comes that  $  \int_{\y} \exp{T_0(x,y)} q_0(\ud y)   = \exp{m(x)}  \int_{\y} \exp{\widetilde{T}_0(x,y)} q_0(\ud y)$ almost surely in $x$. And we finally get $p_0$-almost surely:
\begin{align}
	p_{T_0}(y|x) = \frac{\exp{T_0(x,y)}}{Z(T_0(x))} = \frac{\exp{T_0(x,y) + m(x)}}{\exp{m(x)}Z(T_0(x))} = p_{T_0}(y|x) 
\end{align} 

\end{proof}

\begin{lemma}[Proposition A.3 in \cite{Sriperumbudur:2013}]
\label{lemm:range_space}
	Let $C$ be a bounded, positive self-adjoint compact operator on a separable Hilbert space $\h$. For $\lambda >0 $ and $ T \in \h$, define $ T_{\lambda} := (C + \lambda I )^{-1}CT $ and $ \mathcal{A}_{\theta}(\lambda) := \Vert C^{\theta} (T_{\lambda} - T)  \Vert_{\h}  $ for $\theta \geq 0 $. Then the following hold.
	\begin{enumerate}
		\item For any $\theta>0$,  $\mathcal{A}_{\theta}(\lambda) \rightarrow 0$ as $ \lambda \rightarrow 0$ and if $ T \in \overline{\mathcal{R}(C)} $, then $\mathcal{A}_{0}(\lambda) \rightarrow 0$ as $ \lambda \rightarrow 0$.
		\item If $T \in  \mathcal{R}(C^{\beta})$ for $\beta \geq 0 $ and $ \beta + \theta >0 $, then 
		 \begin{align}
		 	 \mathcal{A}_{\theta}(\lambda ) \leq \max \{  1, \Vert C  \Vert^{\beta + \theta -1} \}\lambda^{\min \{ 1, \beta + \theta \} }\Vert C^{-\beta}T\Vert_{\h} 
		 \end{align}
	\end{enumerate}	

\end{lemma}

\begin{proof}
	
	\begin{enumerate}
		\item Since $C$ is bounded, compact and positive self-adjoint, Hilbert-Shmidt  and $\h$ is a separable Hilbert space then $C$ admits an Eigen-decomposition of the form $ C = \sum_l \alpha_l \phi_l\langle \phi_l \rangle_{\h} $ where $ ( \alpha_l )_{l\in \n} $ are positive eigenvalues and $(\phi_l)_{l\in \n} $ are the corresponding unit eigenvectors that form an ONB for $\mathcal{R}(C)$. Let $\theta = 0$. Since $T  \in \overline{\mathcal{R}(C)}$, 
		\begin{align}
			\mathcal{A}^{2}_0(\lambda) 			
			&=
			\Vert  (C+\lambda I)^{-1}C T - T \Vert^2_{\h} 
			= 
			\Big\Vert  \sum_i \frac{\alpha_i}{\alpha_i + \lambda}  \langle T, \phi_i \rangle_{\h}\phi_i - \sum_i \langle T, \phi_i \rangle_{\h}\phi_i \Big\Vert^{2}_{\h}\\
			&=
			\Big\Vert  \sum_i \frac{\lambda}{\alpha_i + \lambda}  \langle T, \phi_i \rangle_{\h}\phi_i \Big\Vert^{2}_{\h} 
			= 
			\sum_i \Big(\frac{\lambda}{\alpha_i + \lambda} \Big)^2 \langle T, \phi_i \rangle_{\h}^2 \rightarrow 0 \text{ as } \lambda \rightarrow 0 
			\end{align}
		by the dominated convergence theorem. For any $\theta >0$, we have:
		\begin{align}
			\mathcal{A}^{2}_0(\lambda) 
			&=
			\Vert  C^{\theta}(C+\lambda I)^{-1}C T - C^{\theta}T \Vert^2_{\h} = \Big\Vert  \sum_i \frac{\alpha_i}{\alpha_i + \lambda}  \langle T, \phi_i \rangle_{\h}\phi_i - \sum_i \langle T, \phi_i \rangle_{\h}\phi_i \Big\Vert^{2}_{\h}.
		\end{align}
		Let $ T = T_{R} + T_N $ where $ T_R \in  \overline{\mathcal{R}(C^{\theta})}$, $T_N \in \overline{\mathcal{R}(C^{\theta})}^{\perp} $ if $  0<\theta\leq 1 $ and $T_N \in \overline{\mathcal{R}(C)}^{\perp} $ if $\theta \geq 1  $. Then 
		\begin{align}
			\mathcal{A}^{2}_0(\lambda) 
			&=
			\Vert  C^{\theta}(C+\lambda I)^{-1}C T - C^{\theta}T \Vert^2_{\h}
			=
			\Vert  C^{\theta}(C+\lambda I)^{-1}C T_R - C^{\theta}T_R \Vert^2_{\h}\\
			&=
			\Big\Vert  \sum_i \frac{\alpha_i^{1+\theta}}{\alpha_i + \lambda}  \langle T_R, \phi_i \rangle_{\h}\phi_i - \sum_i \alpha_i^{\theta}\langle T_R, \phi_i \rangle_{\h}\phi_i \Big\Vert^{2}_{\h}\\
			&=
			\Big\Vert  \sum_i \frac{\lambda \alpha_i^{\theta}}{\alpha_i + \lambda}  \langle T_R, \phi_i \rangle_{\h}\phi_i \Big\Vert^{2}_{\h} 
			=
			 \sum_i \Big(\frac{\lambda\alpha_i^{\theta}}{\alpha_i + \lambda} \Big)^2 \langle T_R, \phi_i \rangle_{\h}^2 \rightarrow 0 \text{ as } \lambda \rightarrow 0
		\end{align}
	\item
	
	If $T\in \mathcal{R}(C^{\beta}) $, then there exists $g \in \h$ such that $ T = C^{\beta} g$.  This yields 
	\begin{align}
		\mathcal{A}^{2}_0(\lambda) 
			&=
			\Vert  C^{\theta}(C+\lambda I)^{-1}C T - C^{\theta}T \Vert^2_{\h}
			=
			\Vert  C^{\theta}(C+\lambda I)^{-1}C^{1+\beta} g - C^{\beta + \theta}g \Vert^2_{\h}\\
			&=
			\Big\Vert  \sum_i \frac{\lambda \alpha_i^{\beta + \theta}}{\alpha_i + \lambda}  \langle g, \phi_i \rangle_{\h}\phi_i \Big\Vert^{2}_{\h} 
			=
			 \sum_i \Big(\frac{\lambda\alpha_i^{\beta + \theta}}{\alpha_i + \lambda} \Big)^2 \langle g, \phi_i \rangle_{\h}^2			
	\end{align}
	Suppose $ 0 < \beta + \theta < 1$. Then
	\begin{align}
		\frac{\lambda \alpha_i^{\beta + \theta}}{\alpha_i + \lambda} = \Big(  \frac{\alpha_i}{\alpha_i + \lambda} \Big)^{\beta + \theta}\Big(  \frac{\lambda}{\alpha_i + \lambda} \Big)^{1 - \beta - \theta} \lambda^{\beta + \theta} \leq \lambda^{\beta + \theta} 
	\end{align}
	On the other hand, for $ \beta + \theta \geq 1 $, we have:
	\begin{align}
			\frac{\lambda \alpha_i^{\beta + \theta}}{\alpha_i + \lambda} 
			=
			 \Big(  \frac{\alpha_i}{\alpha_i + \lambda} \Big)\alpha_i^{\beta + \theta -1} \lambda \leq \Vert \Vert^{\beta + \theta - 1} \lambda.
			\end{align}
			Using the above bounds yields the result.
	\end{enumerate}
\end{proof}

\begin{lemma}[Proposition A.4 in \cite{Sriperumbudur:2013}]
	\label{thm:rates}
	Let $\x$ be a topological space, $\h$ be a separable Hilbert space and $\mathcal{L}_2^{+}(\h) $ be the space of positive, self-adjoint Hilbert-Schmidt operators on $\h$. Define $ R := \int_{\x} r(x)d\mathbb{P}(x) $ and $ \hat{R} := \frac{1}{n} \sum_{a= 1}^m r(X_a) $ where $\mathbb{P}\in M_{+}^1(\x)$ is a positive measure with finite mean,  $  (X_a)_{a=1}^m \sim \mathbb{P}$ and $r$ is an  $\mathcal{L}_2^{+}(\h) $-valued measurable function on $\x$ satisfying $ \int_{\x} \Vert  r(x) \Vert_{HS}^2 d\mathbb{P}(x) < \infty $. Define $ g_{\lambda} := (R+\lambda I )^{-1}Rg  $ for $g\in \h$, $\lambda >0$ and 
		$ \mathcal{A}_0(\lambda) :=  \Vert g_{\lambda} - g \Vert_{\h} $. Let $\alpha \geq 0 $ and $\theta \geq 0$. Then the following hold:
	\begin{enumerate}
		\item $\Vert (\hat{R} - R )(g_{\lambda} - g)  \Vert_{\h} =  O_{\mathbb{P}} (\frac{\mathcal{A}_0(\lambda)}{\sqrt{m}} )$
		\item  $\Vert R^{\alpha}(R + \lambda I )^{-\theta}  \Vert \leq \lambda^{\alpha- \theta}$.
		\item  $\Vert \hat{R}^{\alpha}(\hat{R} + \lambda I )^{-\theta}  \Vert \leq \lambda^{\alpha- \theta}$.
		\item $\Vert (R + \lambda I )^{-\theta}(\hat{R} - R ) \Vert =  O_{\mathbb{P}} (\frac{1}{\sqrt{m\lambda^{2\theta}}} )$.
	\end{enumerate}

\end{lemma}
\begin{proof}
	\begin{enumerate}
		\item Not that for any $f\in \h $,
		\begin{align}
			\E_{\PP}\norm{(\hat{R} - R)f}^2_{\h} = \E_{\PP} \norm{\hat{R}f}_{\h}^2  + \norm{Rf}^2_{\h} - 2\E_{\PP} \inner{\hat{R}f, Rf}_{\h}
		\end{align} 
		where $\E_{\PP} \inner{\hat{R}f,Rf }_{\h}  =  \frac{1}{n} \sum_{a=1}^n \E_{\PP} \inner{r(X_a)f, Rf )}_{\h}  = \frac{1}{n} \sum_{a=1}^n \E_{\PP}\inner{r(X_a), f\otimes Rf}_{HS}   $. Since $ \int_{\x} \norm{r(x)}^2_{HS} d\PP(x)  < \infty  $, $r(x)$ is $\PP$-integrable in the Bochner sense (see  \cite{Retherford:1978} ), and therefore it follows $ \E_{\PP}\inner{r(X_a), f\otimes Rf}_{HS}  =  \inner{ \int_{\x} r(x) d \PP(x), f\otimes Rf  }_{HS} = \norm{Rf}_{HS}^2 $. Therefore, 
		\begin{align}
			\E_{\PP} \norm{(\hat{R} - R)f}_{\h}^2 =  \E_{\PP} \norm{\hat{R}f}^2_{\h} - \norm{Rf}^2_{\h}
 		\end{align}  
 		where 
 		\begin{align}
 			\E_{\PP} \norm{\frac{1}{m} \sum_{a=1}^m  r(X_a)f}_{\h}^2 =  \frac{1}{m^2} \sum_{a,b=1}^m \E_{\PP}  \inner{r(X_A)f , r(X_b)f }_{\h}.
  		\end{align}

  		Splitting the sum into two parts (one with $a = b$ and the other with $a \neq b$), it is easy to verify that $\E_{\PP} \norm{\hat{R}f}_{\h}^2  =\frac{1}{m} \int_{\x} \norm{r(x)f}^2_{\h} d \PP(x)  +  \frac{m-1}{m} \norm{Rf}_{\h}^2 $, therefore yielding

  		\begin{align}
  			\E_{\PP}\norm{(\hat{R} - R)f}^2_{\h} = \frac{1}{m} \Big( \int_{\x} \norm{r(x)f}_{\h}^2  d\PP(x)  - \norm{Rf}_{\h}^2 \Big) 
  			&\leq 
  			\frac{1}{m} \int_{\x}\norm{r(x)f}^2_{\h} d\PP(x))\\
  			&\leq
  			\frac{\norm{f}_{\h}^2}{m } \int_{\x} \norm{r(x)}_{HS}^2d\PP(x)
  		\end{align}
 
 		Using $f = g_{\lambda} - g$, an application of Chebyshev's inequality yields the result.
 		 
  		\item $\norm{R^{\alpha}( R+\lambda I )^{-\theta}}  = \sup_{i} \frac{\gamma_{i}^{\alpha}}{(\gamma_i + \lambda)^{\theta}}    = \sup_{i} \big[ (\frac{\gamma_i}{\gamma_i  +\lambda} )^{\alpha} \frac{1}{( \gamma_i + \lambda  )^{\theta  - \alpha}} \big]   				\leq \sup_i  \frac{1}{( \gamma_i + \lambda )^{\theta - \alpha }} \leq \lambda^{\alpha - \theta}$, where $(\gamma_i )_{i\in n}$ are the eigenvalues of $R$. 
  		\item Same as above, after replacing $(\gamma_i )_{i\in \n}$ by the eigenvalues of $\hat{R}$
  		\item Since $ \norm{(R+\lambda I  )^{-\theta}(\hat{R} - R )}  \leq  \norm{(R+\lambda I  )^{-\theta}(\hat{R} - R )}_{HS}^2  $, consider  $\E_{\PP}   \norm{(R+\lambda I  )^{-\theta}(\hat{R} - R )}_{HS}^2$, which using the technique in the proof of $ (1)$, can be shown to be bounded as 
  		\begin{align}
  			\E_{\PP}  \norm{(R+\lambda I  )^{-\theta}(\hat{R} - R )}_{HS}^2 \leq
  			\frac{1}{m}  \int_{\x}  \norm{(R+\lambda I  )^{-\theta} r(x)}_{HS}^2 d \PP(x)
  			\label{eq:appendix_eq_1}
  		\end{align}
  		Note that 
  		\begin{align}
  			\norm{(R+\lambda I  )^{-\theta} r(x)}_{HS}^2 
  			&=
  			 \inner{ R+\lambda I  )^{-\theta} r(x), R+\lambda I  )^{-\theta} r(x) }_{HS}\\
  			 &=
  			 \norm{(R+\lambda I  )^{-2\theta}}Tr( r(x) r(x)) =  \norm{(R+\lambda I  )^{-2\theta}} \norm{r(x)}_{HS}^2\\
  			 &\leq
  			 \lambda^{-2\theta} \norm{r(x)}_{HS}^2
  			 \label{eq:appendix_eq_2}
  		\end{align}
  		where the inequality follows from (3). Using \cref{eq:appendix_eq_1}  and \cref{eq:appendix_eq_2}, we obtain
  		\begin{align}
  			\E_{\PP}  \norm{(R+\lambda I  )^{-\theta} r(x)}_{HS}^2  \leq  \frac{1}{m\lambda^{2\theta}}  \int_{\x}  \norm{(R+\lambda I  )^{-\theta} r(x)}_{HS}^2 d \PP(x)
  		\end{align}
  		The result follows by an application of Chebyshev's inequality.
	\end{enumerate}
\end{proof}

\section{Failure case for the score-matching approach}
\label{app:failure_case}
We first recall the expressions of the score and expected conditional score for convenience. If $r$ and $s$ are two densities that are differentiable and positive, then the score objective as introduced in  \cite{Hyvarinen:2005} is given by:
\begin{align}\label{eq:app_score}
	\mathcal{J}(r||s) := \frac{1}{2} \int_{\x} r(x)\Vert \nabla_x \log  r(x) - \nabla_x \log s(x) \Vert^2 dx
\end{align}

If $p_0(y|x)$ and $q(y|x)$ are two conditional densities, then the expected conditional score under some marginal distribution $\pi(x)$ is given by:

\begin{align}\label{eq:app_cond_score}
	J(p_0|q) = \int_{\x} \mathcal{J}(p_0(.|x) q(.|x))\pi(x)dx
\end{align}

The positivity condition of the target density $r$ is crucial to get a well-behaved divergence between $r$ and $s$ in \cref{eq:app_score}. When this condition fails, the score becomes degenerate. For instance, if $r$ is supported on two disjoint sets $A$ and $B$ of $\x$ it can be written in the form:
 \[
 r(x) = \alpha_A p_A(x) + \alpha_B p_B(x)
 \]
where $\alpha_A$ and $\alpha_B$ are non-negative and sum to $1$, and $p_A$ and $p_B$ are two distributions supported on $A$ and $B$ respectively.
In this case, any mixture $s(x) =   \beta_A p_A(x) +  \beta_B p_B(x) $ satisfies $J(r||s) = 0$.

Similarly, for the conditional expected score in \cref{eq:app_cond_score} to be well behaved, the conditional density $p_0(y|x)$ needs to be positive on $\y$ for all $x$ in $\x$. When this condition fails to hold, the same degeneracy happens. Indeed, as shown in \cref{fig:failure_case}, consider $p_0$ of the form:
\[p_0(y|x) =  p_A(y)H(x)  + (1-H(x))p_B(y)\]
where $p_A$ and $p_B$ are supported on two disjoint sets $A$ and $B$ respectively and $H$ denotes the Heaviside step function. For this choice of $p_0$ any mixture $q(y) = \beta_A p_A(y)   + \beta_B p_B(y)$ of $p_A$ and $p_B$ satisfies $J(p_0||q) = 0$. This is because their scores match exactly: $\nabla_y \log p_0(y|x) = \nabla_y \log q(y) $ whenever $p_0(y|x)>0$. Note that in this case $q$ doesn't depend on $x$, which means that this approach might learn a model where $x$ and $y$ are independent while a simple investigation of the joint samples $(X_i,Y_i)$ would suggest the opposite.    

\begin{figure}
\begin{center}
\begin{tikzpicture}
\begin{axis}[%
    ,axis line style = thick
    ,domain=-6:6
    ,samples=100
    ,axis lines=middle
    ,enlargelimits=true
    ,xtick=\empty,ytick=\empty
    ,yticklabels={}
    ] 
    \addplot[scale=0.5,line width = 0.5mm,domain=0:6,color = red] {8*exp(-(\x-3)^2)};
    \addplot[scale=0.5,line width = 0.5mm,domain=-6:0,color = blue] {8*exp(-(\x+3)^2)};
    \addplot[scale=0.5,line width = 0.5mm,domain=-6:6 , color = green] {4*(exp(-(\x+3)^2) + exp(-(\x-3)^2))}; 
\end{axis} 
\end{tikzpicture}
\end{center}
\caption{ A Failure case for the expected conditional score-matching. Here a conditional density of the form $p_0(y|x) =  p_A(y)H(x)  + (1-H(x))p_B(y)$ is considered, where $p_A$ and $p_B$ are supported on two disjoint sets $A \subset \R_{-}^{*} $ and $B \subset \R_{+}^{*}$  and $H$ denotes the Heaviside step function. The red curve and blue curve represent $p_0(y|x>0) = p_A$ and  and $p_0(y|x<=0) = p_B$ respectively, while the green curve represent the mixture $q(y) = \frac{1}{2}(p_A(y)   + p_B(y))$. This is a case where the expected conditional score  fails to separate the two conditional distributions $p_0(y|x)$ and $q(y)$. }
\label{fig:failure_case}
\end{figure}
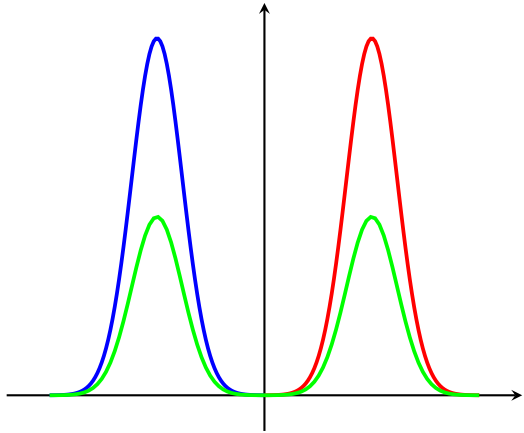
\section{Additional experimental results}
\label{app:additional_exp}
Additional experimental results are shown in \cref{fig:UCI_samples_all_methods} on the Red Wine and Parkinsons datasets. 
\begin{figure}
\centering
\includegraphics[width=\linewidth]{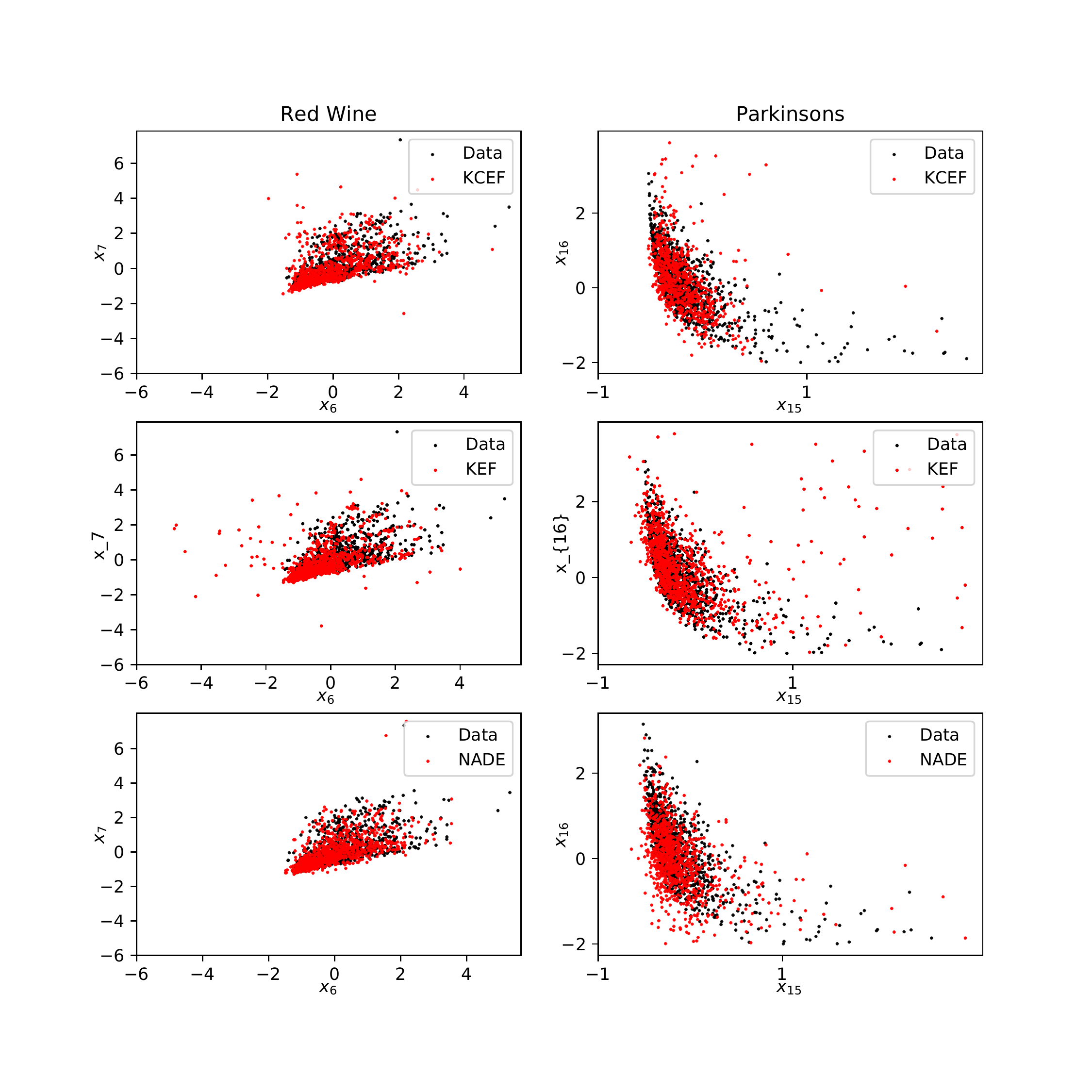}
\caption{Scatter plot of 2-d slices of \textit{red wine} and \textit{parkinsons} data sets, the dimensions are $(x_6,x_7)$ for \textit{red wine} and $(x_{15},x_{16})$ for \textit{parkinsons}. The black points represent 1000 data points from the data sets. In red, 1000 samples from each of the three models KEF, KCEF and NADE.  }
\label{fig:UCI_samples_all_methods}
\end{figure}

Experimental results on the synthetic grid dataset are shown in \cref{fig:old_exp_wave} in the case where an isotropic RBF kernel is used.
  \begin{figure}
\includegraphics[width=0.5\linewidth]{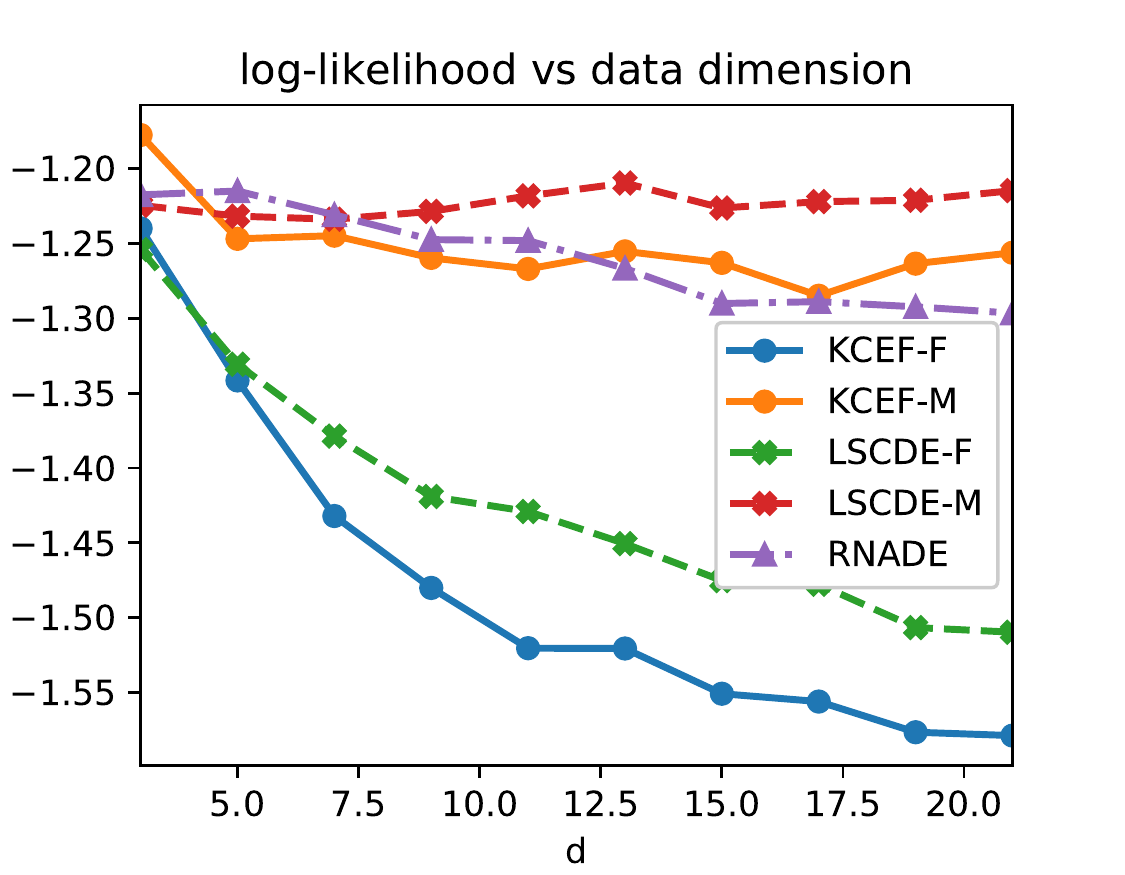}
\centering
\caption{Experimental comparison of proposed method KCEF and other methods (  LSCDE and NADE ) on synthetic \textit{grid} dataset.   log-likelihood per dimension vs dimension,  $N=2000$. The log-likelihood is evaluated on a separate test set of size $2000$.  }
\label{fig:old_exp_wave}
\end{figure}

\end{appendices}

\printbibliography
\end{refsection}

\end{document}